\documentclass{article}

\usepackage{soul}
\usepackage{authblk}
\usepackage{packagesList}

\usepackage[utf8]{inputenc} 
\usepackage[T1]{fontenc}    
\usepackage{hyperref}       
\usepackage{url}            
\usepackage{booktabs}       
\usepackage{amsfonts}       
\usepackage{nicefrac}       
\usepackage{microtype}      
\usepackage{xcolor}         
\usepackage{enumitem}

\theoremstyle{plain}
\newtheorem{theorem}{Theorem}[section]
\newtheorem{lemma}[theorem]{Lemma}
\newtheorem{proposition}[theorem]{Proposition}

\newtheorem{remark}{Remark}[section] 

\theoremstyle{definition}
\newtheorem{definition}[theorem]{Definition}

\newtheorem{example}{Example}[section]

\newtheorem*{remark*}{Remark}
\newtheorem*{discussion*}{Discussion}
\newtheorem{corollary}[theorem]{Corollary}

\usepackage{stmaryrd}

\usepackage{dsfont}

\newcommand{\rset}{\mathbb{R}}

\newcommand{\ind}{\mathds{1}}

\newcommand{\ffrac}[2]{\ensuremath{\frac{\displaystyle #1}{\displaystyle #2}}}
\newcommand{\un}[1]{\ind{\left\{#1\right\}}}
\newcommand{\indic}{\un} 
\newcommand{\PP}[1][]{\ifthenelse{\equal{#1}{}}{\ensuremath{\mathbb{P}}}{\ensuremath{\mathbb{P}\left( #1 \right) }}}
\newcommand{\EE}[1][]{\ifthenelse{\equal{#1}{}}{\ensuremath{\mathbb E}}{\ensuremath{{\mathbb E}\left[ #1 \right]}}}
\newcommand{\Var}[1][]{\ifthenelse{\equal{#1}{}}{\ensuremath{\mathrm{Var}}}{\ensuremath{{\mathrm{Var}}\left[ #1 \right]}}}
\newcommand{\Cov}[1][]{\ifthenelse{\equal{#1}{}}{\ensuremath{\mathrm{Cov}}}{\ensuremath{{\mathrm{Cov}}\left[ #1 \right]}}}

\DeclareMathOperator*{\argmin}{arg\,min}

\DeclareMathOperator{\sign}{sign}

\usepackage{xspace}

\newcommand\ie{\emph{i.e.}\xspace}

\usepackage{xcolor}

\newcommand{\EVT}{\textsc{EVT}\xspace}

\newcommand{\ERM}{\textsc{ERM}\xspace}

\newcommand{\ML}{\textsc{ML}\xspace}
\newcommand{\VC}{\textsc{VC}\xspace}

\newcommand\eg{\emph{e.g. }\xspace}

\newcommand{\ER}[1][]{\widehat{ \mathcal{R}}}
\newcommand{\risk}{\mathcal{R}}

\title{Sharp error bounds for imbalanced classification: how many examples in the minority class?}
\author[1]{Anass Aghbalou}
\author[2]{Fran\c{c}ois Portier}
\author[3]{Anne Sabourin}

\affil[1]{LTCI, T\'el\'ecom Paris, Institut Polytechnique de Paris, LTCI, Palaiseau,
	France}

\affil[2]{Ensai, CREST, Rennes, France}
\affil[3]{ Université Paris Cité, CNRS, MAP5, Paris, France }
\begin{document}

\maketitle

\begin{abstract}
%
%

When dealing with imbalanced classification data, reweighting the loss function is a standard procedure allowing to equilibrate between the true positive and true negative rates within the risk measure.
Despite significant theoretical work in this area, existing results do not adequately address a main challenge within the imbalanced classification framework, which is the negligible size of one class in relation to the full sample size and the need to rescale the risk function by a probability tending to zero. To address this gap, we present two novel contributions in the setting where the rare class probability approaches zero: (1) a non asymptotic fast rate probability bound for constrained balanced empirical risk minimization, and (2) a consistent upper bound for balanced nearest neighbors estimates. Our findings provide a clearer understanding of the benefits of class-weighting in realistic settings, opening new avenues for further research in this field.

\end{abstract}

\section{Introduction}

Consider the problem of binary classification with covariate $X$ and
target $Y\in\{-1,1\}$. 
The flagship
approach to this problem in statistical learning is Empirical Risk
Minimization (ERM), which produces approximate minimizers of $
\mathcal{R} (g) = \EE [ \ell(g(X),Y)] $, given a loss function $\ell$
and a family of candidate classifiers $g\in\mathcal{G}$, with the help
of observed data. 
with classifier $g$, $\ell_g(X,Y) = \ell( g(X), Y)$.  However, when
the underlying distribution is imbalanced, that is $p = \PP (Y = +1)$
is relatively small, minimizing empirical version of $\mathcal{R}$
often leads to trivial classification rules for which the majority
class is always predicted,  because minimizing $
\mathcal{R} (g) $ in that case is similar to minimizing $ \EE [
\ell(g(X),Y) \,|\, Y=-1]$.  Indeed by the law of total probabilities,
$\mathcal{R}(g) = p \EE [ \ell(g(X),Y) \,|\, Y=+1] + (1-p) \EE [
\ell(g(X),Y) \,|\, Y=-1]$ and the former term is negligible with
respect to the latter when $p\ll 1$.  For this reason, even though
standard ERM  approaches might enjoy satisfactory generalization
properties over imbalanced distributions, with respect to the standard
risk $\mathcal{R}$, they may lead to unpleasantly high false negative
rates 
and in general the average error
on the minority class has no reason to be small, as its contribution
to the overall risk $\mathcal{R}$ is negligible. This is typically
what should be avoided in many applications when false negatives are
of particular concern, among which medical diagnosis or anomaly
detection for aircraft engines, considering the tremendous  cost of an
error regarding a positive example.

Bypassing the  shortcoming described above is the main goal of  many  works regarding imbalanced classification. The existing literature may be roughly divided into oversampling approaches such as SMOTE and GAN  \citep{chawla2002smote,mariani2018bagan}, undersampling procedures \citep{Liu2009,triguero15} and risk balancing procedures also known as \emph{cost-sensitive learning} \citep{Scott2012,xu2020class}. Here we focus on the latter approach which enjoys  numerous benefits, including simplicity, improved decision-making \citep{Elkan2001,VIAENE2005}, improved class probability estimation \citep{wang19,fu2022cost},  better resource allocation \citep{xiong15,ryu2017} and increased fairness \citep{menon2018cost,agarwal2018reductions}. By incorporating the varying costs of misclassification into the learning process, it enables models to make more informed and accurate predictions for the minority class, leading to higher-quality predictions.  
 Balancing the risk consists of 
minimizing risk measures that differ significantly from the standard empirical risk, by means of an appropriate weighting of the negative and positive errors,  in order to achieve a  balance between the contributions of the positive and negative classes to the overall risk.  
In the present paper our main focus is  the  balanced-risk,
$ \risk_{p}(g)  =  \EE [ \ell(g(X),Y) \,|\,  Y=+1] +   \EE [ \ell(g(X),Y) \,|\,  Y=-1]$.  
Other metrics might be considered as detailed for instance in Table~1 in~\cite{Menon2013} which we do not analyze here for the sake of conciseness, even though our techniques of proof may be straightforwardly extended to handle  these variants.

Empirical risk minimization based on the balanced risk is a natural idea, which is widely exploited by practitioners and has demonstrated its practical relevance in  several operational contexts \citep{elkan2001foundations,SUN2007,wang16,khan2018,PATHAK2022}. From a theoretical perspective, class imbalance has
been the subject of several works. For instance, the consistency of the resulting classifier is investigated in  \cite{koyejo2014consistent}. Several different risk measures and  loss functions are considered in \cite{Menon2013} where results of asymptotic nature are established, for fixed $p>0$, as  $n\to\infty$. Also in the recent work by \cite{xu2020class}, generalization bounds are established for the imbalanced multi-class problem for a robust variant of the balanced risk considered here. Their main results from the perspective of class imbalance, is their  Theorem~1 where the upper bound on the (robust) risk includes a term scaling as $1/( p \sqrt{n})$. 
 A related subject 
is weighted ERM  where the purpose is to learn from biased data 
(see \emph{e.g.} \cite{vogel2020weighted,bertail21a} and the references therein), that is, the training distribution and the target distribution differ. The imbalanced classification problem may be seen as a particular instance of this transfer learning problem, where the training distribution is imbalanced and the target is a balanced version of it with equal class weights. A necessary assumption in \cite{bertail21a} is that the density of the target with respect to the source is bounded, which in our context is equivalent to requiring that $p$ is bounded away from $0$, an explicit assumption in  \cite{vogel2020weighted} where the main results impose that $p>\epsilon$ for some fixed $\epsilon>0$.

The common working assumption in the cited references that $p$ is bounded from below, renders their application disputable in concrete situations where the number of positive examples is negligible with respect to a wealth of negative instances. To our best knowledge the literature is silent regarding such a situation. 
More precisely, we have not found neither asymptotic results covering the case where $p $ depends on $n$ in such a way that $p\to 0$ as $n\to \infty$; nor finite sample bounds which would remain sharp even in situations where $p$ is much smaller than $1/\sqrt{n}$. Such situations arise in many examples in machine learning (see \eg the motivating examples in the next section).    However, existing works assume that  
the sizes of  both classes are of comparable  magnitude, which leaves a gap  between theory and practice. A possible explanation is that existing works do not exploit the full potential of the \emph{low variance} of the loss functions on the minority class typically induced by boundedness assumptions combined with a low expected value associated with a small $p$.  
%

It is the main purpose of this work to overcome this bottleneck and obtain generalization guarantees for the balanced risk which remain sharp even for very small $p$, that is, under sever class imbalance. 
Our purpose is to obtain upper bounds on the deviations of the empirical risk (and thus on the empirical risk minimizer) matching the state-of-the art, up to replacing the sample size $n$ with $np$,  the mean size of the rare class.  To our best knowledge, the theoretical results which come closest to this goal  are normalized Vapnik-type inequalities (Theorem 1.11 in \cite{lugosi2002pattern}) and relative deviations (Section~5.1 in~\cite{boucheron2005theory}). However the latter results  only apply to binary valued functions and as such do not extend immediately to general real valued loss functions which we consider in this paper, nor  do they yield fast rates for \emph{imbalanced} classification problems, although relative deviations play a key role in establishing fast rates in \emph{standard} classification as reviewed in Section~5 from \cite{boucheron2005theory}. Also, 
as explained above, we have not found any theoretical result regarding  imbalanced classification which would leverage these bounds in order to obtain guarantees with  leading terms depending on  $np$ instead of $n$. 

Our main tools are $(i)$ 
Bernstein-type concentration inequalities (that is, upper bounds including a variance term) for empirical processes that are consequences of Talagrand inequalities 
such as in \cite{gine2001consistency}, $(ii)$ fine  controls of the expected deviations of the supremum error in the vicinity of the Bayes classifier, by means of local Rademacher complexities  \cite{Bartlett2005,bartlett2006empirical}.   
%
Our contributions are two-fold. \\
\quad {\bf 1. } We establish an estimation error bound on the balanced risk which holds true for  VC classes of functions, which  scales as $1/\sqrt{np}$ instead of the typical rate $1/\sqrt n$ in  well-balanced problem, or $1/(p\sqrt{n})$ in existing works regarding the imbalanced case  (\emph{e.g.} as in \cite{xu2020class}).   Thus,  in practice,  our setting encompasses the case where $p\ll 1$ (severe class imbalanced) and our upper bound constitutes  a   crucial improvement by a factor $\sqrt{p}$ compared with  existing works in imbalanced classification. 
Applying the previous bound to the $k$-nearest neighbor classification rule, we obtain the following new consistency result: as soon as $kp$ goes to infinity, the nearest neighbors classification rule is consistent in case of relative rarity.  \\
\quad {\bf 2. } We obtain fast rates for  empirical risk minimization procedures under an additional classical assumption called a Bernstein condition. Namely we prove  upper bounds on the excess risk  scaling as  $1/(np)$, which matches  fast rate results in the standard, balanced case, up to replacing the full sample size $n$ with the expected minority class size $np$. To our best knowledge such  fast rates  are the first of their kind in the imbalanced classification literature. 



\noindent \textbf{Outline.} Some mathematical background about imbalanced classification and some motivating examples are given in Section 2. In Section 3, we state our first non-asymptotic bound on the estimation error over VC class of functions and consider application to $k$-nearest neighbor classification rules.  In Section 4, fast convergence rates are obtained and an application to ERM is given. Finally, some numerical experiments are provided in Section 5 to illustrate the theory developed in the paper. All proofs of the mathematical statements are in the supplementary material.


\section{Definition and notation}\label{sec:background}

Consider a standard binary classification problem where random covariates $X$, defined over a space $\mathcal X$, are employed to distinguish between two classes defined by their labels $Y = 1$ and $Y = -1$. The underlying probability measure is denoted by $\mathbb{P}$ and the associated expectancy, by $\mathbb{E}$. The  law  of $(X,Y)$   on the sample space $\mathcal X\times \mathcal Y := \mathcal X\times \left\{-1,1\right\}$, is denoted by $P$. We assume that the label $Y = 1$ corresponds to  minority class, i.e., $p= \mathbb{P}(Y=1) \ll 1$. In the sequel we assume that $p>0$, even though $p$ may be arbitrarily small. 

We adopt notation from empirical process theory. Given a measure $\mu$ on $\mathcal X\times \mathcal Y$ and a real function $f$ defined over $\mathcal X\times \mathcal Y$, we denote $\mu(f) = \int f d\mu$. When  $f = \ind_C$ for a measurable set $C$, we may write interchangeably $\mu(f) = \mu(\ind_C) = \mu(C)$. We denote by $P_+$ the conditional law of $(X,Y)$ given that $Y=+1$, thus
$$P_{+}(f) = \ffrac{\mathbb{E}(f(X,Y) \indic{Y=1})}{p}=\mathbb{E}(f(X,Y)\mid Y=1).$$
In addition, we denote by  $\Var_{+}(f)$ the conditional variance of  $f(X,Y)$ given that $Y=+1$. The conditional distribution and variance $P_-$ and $\Var_-$ are defined similarly, conditional to $Y=-1$.  
Consider more generally the weighted probability measures, for $q\in(0,1)$,
$$
P_q( f ) = \frac{1}{2} \big( q^{-1} P( f I_+ ) + (1-q)^{-1} P( f I_-) \big),   
$$
where $ fI_s = f(x,y)\un{y=s.1}$, $s\in\{+,-\}$.  Notice that
$P_p f =  (P_+ f + P_- f)/2$. 


In this paper we consider general discrimination functions (also
called \emph{scores}) $g: \mathcal{X}\to \rset$ and loss functions
$\ell:\rset\times \{-1,1\}\to \rset$, and our results will hold under
boundedness and Vapnik-type complexity assumptions detailed below in
Sections~\ref{sec:standard},~\ref{sec:fast}. Given a score function
$g$ and a loss $\ell$, it is convenient to introduce the function
$\ell_g: (x,y)\mapsto \ell(g(x), y)$. Thus the 
(unbalanced) risk of the score function $g$ is
$\risk(g) = \mathbb{E}[\ell_g\left(X,Y\right)]$. Notice that the
standard $0-1$ misclassification risk, $\risk^{\mathrm{0-1}}(g)= \PP[g(X)\neq Y]$,  is retrieved when 
$g$ is real valued and $\ell(g(x),y) = \mathrm{sign}(-g(x)y)$. Allowing for more general scores and losses is a standard
approach in statistical learning allowing to bypass the NP-hardness of
the minimization problem associated with
$\risk^{\mathrm{0-1}}$. Typically (although not required here) 
$\ell_g(x,y)=\phi\left(-g(x)y\right)$, where $\phi$ is convex and
differentiable with $\phi^\prime(0)<0$
\citep{zhang2004statistical,bartlett2006}. This ensures that the loss
is classification calibrated and that
$\risk(g) = \EE[\ell_g\left(X,Y\right)]$ is a convex upper bound  of
$\risk^{\mathrm{0-1}}(g)$. Various consistency results ensuring
that
$g^*=\argmin_{g\in\mathbb{R}^{\mathcal{X}}}\risk(g)=\argmin_{g\in\mathbb{R}^{\mathcal{X}}}\risk^{\mathrm{0-1}}(g)$
can be found in \cite{bartlett2006}. Examples include the logistic
($\phi(u)=\log(1+e^{-u})$), exponential ($\phi(u)=e^{-u}$), squared
($\phi(u)=(1-u)^2$), and hinge loss
($\phi(u)=\max(0,1-u)$).  

The balanced $0-1$ risk, defined as the arithmetic mean 
 $\risk_{p}^{0-1}(g) =  (P_{+} (Y\neq g(X))     + P_{-}  ( Y\neq g(X)))/2$
is called the \emph{AM risk} in \cite{Menon2013}. The minimizer of the latter risk, $g^*_{p}$, is known as the balanced Bayes classifier. It returns $1$ when $\eta(X)  = \PP (Y= + 1 \,|\, X) \geq p$ and $-1$ otherwise (see \emph{e.g.} Th.~2  or Prop.~2 in
\cite{koyejo2014consistent}).  Here we consider general  weighted  risks and  real-valued loss function $\ell_g$, defined  for  $g\in \mathcal G$, $q\in(0,1)$  as 
$$
\risk_{q}(g) = P_q(\ell_q).  $$
Of particular interest is the case $q=p$, for which $\risk_p$ is called the balanced risk. 

Given an independent and identically distributed sample  $(X_i,Y_i)_{1\leq i\leq n}$ according to $P$, we denote by $P_n$ the empirical measure, 
$ P_n (f) =  (1/n) \sum_{i=1} ^ n f(X_i,Y_i),  $
 for any  measurable and real-valued function $f$ on $\mathcal X \times \mathcal Y$. 
 While the standard empirical probability  is simply expressed as $P_n(f)$ for any measurable function $f$, the weighted empirical  probability with weight $q\in(0,1)$ is
 $$
P_{n,q}(f) = \frac{1}{2}\big(q^{-1} P_n(f I_+) + (1-q)^{-1}P_n(f I_-)  \big). 
 $$
 The   balanced empirical probability $ P_{n, \hat p }(f) $  must be defined in terms of $\hat p$ if $p$ is unkown. 
 We shall sometimes use that $P_{n, \hat p}(f) = ( P_{n, +} f + P_{n, -} f )/ 2$, where  
$  P_{n,+} (f ) =  \hat p ^{-1} P_n (f I_+)$ (by convention we set $P_{n,+} (f ) = 0$ when $\hat p =P_n (Y=1) = 0$). The empirical measure of the negative class, $P_{n,-}$, is defined in a similar manner.
For $q \in(0,1)$ the weighted $q$-empirical risk is
$$
\risk_{n,q}(g) = P_{n,q}(\ell_g). 
$$
Finally  the balanced empirical risk considered in this paper is 
$$\risk_{n,\hat p}(g) = P_{n,\hat p}(\ell_g) = \frac{1}{2}\left( P_{n,+} ( \ell_g    ) + P_{n,-}  (   \ell_g   )\right).$$ For simplicity we make the standard assumption throughout this paper that for all $q\in(0,1)$, a minimizer $g^*_q$ (resp. $\hat g_q$)  of $\risk_{q}$ (resp. $\risk_{n,q}$) exists, in particular that
$g_p^*$  (resp. $\hat g_{\hat p}$), a  minimizer of the  balanced  risk
$\risk_p$(resp.  $\risk_{n, \hat p}$), exists. 

\paragraph{Motivating examples} We now present two examples where the probability $p\to 0$ as $n\to \infty$ :
\begin{enumerate}
\item The first example is the problem of contaminated data which is central in the robustness literature. A common theoretical assumption is that the number of anomalies $n_0$ grows sub-linearly with the sample size, as discussed in \citep{xu2012outlier,staerman21a}. In this context, $n_0=n^a$ for some $a<1$ and consequently, $p=n^{a-1} \to 0$.
\item The second example pertains to Extreme Value Theory (\EVT) \citep{resnick2013extreme,goix15,jalalzai2018binary,aghbalou2023tail}. Consider a continuous positive random variable $T$, predicting exceedances over arbitrarily high threshold $t$ may be viewed as a binary classification problem. Indeed for fixed $t$,
consider the binary target $Y = \un{T > t}-\un{T \leq t}$ with marginal class probability $p= P (T > t)$. The
goal is thus to predict $Y$, by means of the covariate vector $X$. 
In practice, \EVT based approaches set the threshold $t$ as the $1-\alpha$ quantile of $T$ with $\alpha=k/n \to 0$ and $k=o(n)$. This approach essentially assumes that the positive class consists of the $k=o(n)$ largest observations of  $T$ so that $P(T>t)=P(Y=1)=k/n\to 0$.
\end{enumerate}
The considered framework ``$p$ going to $ 0$ with $n$'' is also valuable  from a theoretical perspective as it allows to picture a learning frontier defined in terms of the relative order of magnitude of  $p$ and $n $, above which learning is consistent while below this threshold, one would not be able to estimate the quantity of interest. In a similar spirit, the in  high-dimensional statistics, it is customary to let  the underlying dimension to grow with the sample size.


\section{Standard learning rates under relative rarity}\label{sec:standard}
\subsection{A First Deviation Inequality for Balanced Risks}
The primary goal of this paper is to assess the error associated with
estimating the balanced risk $\risk_{p}(g)$ using the
empirical balanced risk $\risk_{n,\hat p}(g)$. Given the
definition of the balanced risk, the quantity of interest takes the
form $(P_{n,+} - P_+) (f) $, and a similar analysis applies to
$(P_{n,-} - P_-) (f)$. In this paper we control the complexity of
the function class \emph{via} the following notion of
VC-complexity. Let $(S, \mathcal S)$ be a measurable space. Let $\mathcal F$ be a class of real valued functions defined on $S$ and $Q$ be a probability measure on $(S,\mathcal S)$. Given $\mathcal F \subset L_2(Q)$, i.e., $Q(f^2)< \infty$, for each $ f\in \mathcal F$, the $\epsilon$-covering number, denoted by $\mathcal N (\mathcal F , L_2(Q) , \epsilon)$, is defined as the smallest number of closed $L_2(Q)$-balls of radius $\epsilon > 0$ needed to cover $\mathcal F$. For a given class of functions $\mathcal F$, $F$ is called an envelope if $ |f(x)| \leq F(x) $, for all $x\in S$ and all $f\in \mathcal F$.

\begin{definition}\label{def:VC-class}
 The family of functions $\mathcal F$ is said to be of VC-type with constant envelope $U>0$ and parameters $v\geq 1 $  and $A\geq 1$ if {all functions in} $\mathcal F$ are bounded by $U$ and for any $0 < \epsilon < 1$ and any probability measure $Q$ on $(S, \mathcal S)$, we have
\begin{equation*}
\mathcal N \left(\mathcal F,  L_2(Q)  ,  \epsilon  U  \right) \le (A/\epsilon)^{v}.
\end{equation*}
\end{definition}

The connection between the usual \VC definition \citep{vapnik71} and Definition~\ref{def:VC-class} can be directly established through Haussler's inequality \citep{HAUSSLER1995}, which indicates that the covering number of a class of binary classifiers with VC dimension $v$ (in the sense of \cite{vapnik71}) is given by 
$$ \mathcal N\left( \mathcal F, L_2(Q),\epsilon \right) \leq e (v+1) (2e /\epsilon^2 )^v \leq  (e^2  /\epsilon )^{2v} . $$
Thus a VC-class of functions in the sense of \cite{vapnik71} is necessarily  a VC-type class in the sense of Definition~\ref{def:VC-class}. 

Notice that within a class  $\mathcal{F}$ with envelope $U>0$, the following variance bounds are automatically satisfied:
$$\sigma_{+} ^2,\sigma_{-} ^2  =\sup_{f\in \mathcal F} \Var_+(f),\sup_{f\in \mathcal F} \Var_-(f)\leq U^2.$$
The following theorem states a uniform generalization bound that incorporates the probability of each class in such a way that the deviations of the empirical measures are controlled by the expected number of examples in each class, $np$ and $n(1-p)$. Interestingly the deviations may be small even for small $p$, as soon as the product $np$ is large. 
The bound also incorporates the conditional variance of a class ($\sigma_+^2,\sigma_-^2$), which will play a key role in our application to 
nearest neighbors.

\begin{theorem}\label{theo:VC-standard-rate}
	Let $\mathcal F$ be of VC-type with constant envelope $U$ and parameters $(v,A)$.  For any $n$ and $\delta$ such that 
	$$ n p \geq   
	\max\left[\frac{U^2}{\sigma_+^2} v   \log\left( \frac{ K A U } {  \delta \sigma_+ \sqrt{p}} \right), 8\log(1/\delta)\right] $$ 
	we have with probability $1-2\delta$, 
	$$ \sup_{f\in \mathcal F} \left| P_{n,+} ( f ) - P_+(f )\right| \leq  K  \sigma_+ \sqrt{\frac{ v }{np} \log\left( \frac{ K A U } {  \delta \sigma_+ \sqrt{p}} \right)    }   $$
	for some universal explicit constant $K>0$. 
\end{theorem}


\begin{remark}
This upper bounds extends Theorem 1.11 in \cite{lugosi2002pattern}, which is limited to  a binary class of functions characterized by finite shatter coefficients. The extension is possible  by utilizing Berstein type inequalities for empirical processes defined on VC-type classes of functions as in  \cite{gine2001consistency,portier2021nearest}. It is crucial to recognize that most existing non-asymptotic statistical rates in the imbalanced classification literature \citep{Menon2013,koyejo2014consistent,xu2020class} follow the rate $1/ (p\sqrt{n})$, leading to a trivial upper bound when $p\leq 1/\sqrt{n}$. In our analysis, the upper bound remains consistent provided that $np \to \infty$, thereby emphasizing the merits of using concentration inequalities incorporating  the  variance of the positive class 
$\Var(f\un{y=1})\leq Up \ll 1.$
\end{remark}
The next corollary, which derives from Theorem~\ref{theo:VC-standard-rate} together with standard arguments, provides  generalization guarantees for ERM algorithms based upon the balanced risk. Namely it gives an upper bound on the excess risk of a minimizer of the  balanced risk. The proof is provided in the supplementary material for completeness. 
\begin{corollary}\label{prop_min_risk}
	Suppose that $\{\ell_g\, :\, g\in \mathcal G\}$ is VC-type with envelope $U$ and parameters $(v,A)$.  Under the conditions of Theorem \ref{theo:VC-standard-rate} and that $p  \leq 1/2$, we have, with probability $1-\delta$,
	\begin{align*}
		\risk_{p}(\hat g_{p}) \leq  \risk_{p}&(g^{*}_{p}) +  K  \sigma_{\max} \sqrt{\frac{ v }{np} \log\left( \frac{ K A U } {  \delta  \sigma_{\min} \sqrt{p}} \right)    } ,  
	\end{align*}	
        with $ \sigma_{\max} = \sigma_- \vee \sigma_+$, $ \sigma_{\min} = \sigma_- \wedge \sigma_+$ and $K>0$ is an explicit universal 
        constant ($\sigma_-$ and $\sigma_+$ are defined as before but with  $\{\ell_g\, :\, g\in \mathcal G\}$ instead of $\mathcal F$).
\end{corollary}

The previous result shows that whenever $ np \to \infty$,  learning from \ERM based on a VC-type class of functions is consistent. 
Another application of our result 
pertains to  $k$-nearest neighbor classification algorithms. In this case the sharpness of our bound is fully exploited by  leveraging the variance term $\sigma_+$. This is the subject of  the next section.

\begin{remark}
The AM risk objective may be viewed as a Lagrangian
	formulation of the constrained optimization problem associated
	with the Neyman-Pearson framework \citep{Scott2005}.  For instance, \cite{rigollet2011neyman}
	demonstrate an upper bound of order $1/\sqrt{np}$ under the
	condition that the number of observations in each class remains
	fixed. Their proof technique heavily relies on this fixed sample size assumption for each class. It is unclear how to leverage this result to our context where $n^+$ and $n^-$ are random.   Another key result in the given reference
	(Corollary~6) is also close to our framework, however the stated upper bound holds with probability 
	$1-\delta-\exp(-np^2)$. Thus the probability of the adverse event becomes large as soon as $p\le 1/\sqrt{n}$, rendering the guarantee vacuous. Incidentally, note  that the  analysis in the referenced work relies on  the assumption that the family of classifiers is finite. 
\end{remark}
 
\subsection{Balanced $k$-Nearest Neighbor}

In the context of imbalanced classification, we consider here a balanced version of the standard  $k$-nearest neighbor ($k$-NN for short) rule, which is designed in relation with the 
balanced risk  $R_{bal} ^* (g)$. 
We establish the consistency of the balanced $k$-NN classifier with respect to the balanced risk.

Let $x\in \mathbb R^d$ and $\|\cdot\|$ be the Euclidean norm on $\mathbb R^d$. Denote by $ B(x,\tau)$ the set of points $z\in \mathbb R^d $ such that $\|x-z \|\leq \tau$. For $n\geq 1$ and $k\in\{1,\; \ldots,\; n\}$, the $k$-NN radius at $x$ is defined as
\begin{align*}
\hat \tau_{x} : =  \inf   \left\{ \tau\geq 0 \, :\,  \sum_{i=1}^n 1 _{ B(x,\tau) }(X_i) \geq k \right\} .
\end{align*} 
Let $\hat I(x)$ be the set of index $i$ such that $X_i \in B (x ,\hat \tau_{x} ) $ and define the estimate of the regression function $\eta(x)$ as
 \begin{align*}
 \hat \eta (x)= \frac{1}{k} \sum_{i\in \hat I (x)} \ind_{ Y_i  = 1  }.
 \end{align*}
While standard NN classification rule is a majority vote following $\hat \eta (x)$, i.e., predict $1$ whenever $ \hat \eta (x) \geq 1/2$, it is natural, in view of well known results recalled in Section~\ref{sec:background},  to consider a balanced classifier $\hat g$ for imbalanced data 
 which predicts $1$ whenever $ \hat \eta (x) \geq  \hat p$, that is $\hat g = \sign(  \hat \eta (x) / \hat p -1) $.

 The analysis of the $k$-NN classification rule is conducted for
 covariates $X$ that admit a density with respect to the Lebesgue
 measure. We will need in addition that the support $S_X$ is well
 shaped and that the density is lower bounded away from zero. These  standard
 regularity conditions in the $k$-NN literature are recalled below. 
\begin{enumerate}[label=(X\arabic*) , wide=0.5em,  leftmargin=*]
  \item  \label{cond:reg0} The random variable $X$ admits a density $f_X$ with compact support $S_X \subset \mathbb R^d $.
  \item  \label{cond:reg1} There is $c>0$ and $T>0$ such that $ \forall \tau \in (0,T] $ and $\forall x\in S_X$,
\begin{align*}
\lambda (S_X \cap  B(x, \tau ) ) \geq c \lambda   ( B(x, \tau )) , 
\end{align*}
where $\lambda$ is the Lebesgue measure.
\item \label{cond:reg2} There is $0 < b_X\leq U_X <+\infty$ such that
\begin{align*}
& b_X\leq f_{X}(x) \leq U_X , \qquad \forall x \in S_X .
\end{align*}
\end{enumerate}

In light of Proposition \ref{classification_basic_bound} (stated in the supplement), we consider the estimation of $ \eta(x)/p $ using the $k$-NN estimate $ \hat \eta (x)  / \hat p$. The proof, which is postponed to the supplementary file, crucially relies on arguments from the proof of our  Theorem \ref{theo:VC-standard-rate} combined with known results concerning the VC dimension of Euclidean balls \citep{wenocur1981some}.

\begin{theorem}\label{prop:as_knn_unif}
Suppose that \ref{cond:reg0} \ref{cond:reg1} and \ref{cond:reg2}  are fulfilled and that $x\mapsto \eta (x)  / p $ is $L$-Lipschitz on $ S_X$ ($L$ does not depend on $n$). Then whenever $p\to 0$, $ pn / \log(n) \to \infty$, $k/\log(n) \to \infty$ and $k/n \to 0$,
we have, with probability $1$,
$$\sup _{x\in \mathcal X} \left| \frac{\hat \eta (x)}{\hat p } - \frac{\eta(x)  } { p} \right|     = O  \left( \sqrt{ \frac{\log(n)}{kp}} + \left(\frac{k}{n}\right)^{1/d} \right)  . $$
\end{theorem}

 The consistency of the balanced $k$-NN with respect to the AM risk, encapsulated in the next corollary,  follows from Theorem~\ref{prop:as_knn_unif} combined with an additional result (Lemma~\ref{classification_basic_bound}) relating  the deviations of the empirical regression function with the excess balanced  risk.

\begin{corollary}\label{coro:knn-consistency}
Suppose that \ref{cond:reg0} \ref{cond:reg1} and \ref{cond:reg2}  are fulfilled and that $x\mapsto \eta (x)  / p $ is $L$-Lipschitz on $ S_X$. Then whenever $p\to 0$, $k p /\log(n) \to \infty$ and $k/n \to 0$,
we have, with probability $1$,
$$\risk_{p}(\hat g_{\hat p}) \rightarrow   \risk_{p}(g^*_p) .$$
\end{corollary}
The main interest of Corollary~\ref{coro:knn-consistency} is that  the condition for consistency involves the product of the number of neighbors $k$ with the rare class probability $p$. The take-home message is that  learning nonparametric decision rules is possible with imbalanced data,  as soon as $kp$ is large enough. In other words 
local averaging process should be   done carefully to ensure a  sufficiently large  \emph{expected} number of neighbors from the rare class.


\section{Fast rates  under relative rarity}\label{sec:fast}

We now  state and prove  a concentration inequality that is key to obtain  fast convergence rates for the  excess risk in the context of balanced ERM. The following condition regarding a class of functions $\mathcal{F}$  is a prevalent concept within the fast rates literature \citep{bartlett2006empirical,klochkov2021stability}.
A class of function $\mathcal{F}$ satisfies a Bernstein condition relative to a    probability measure $P$ on $\mathcal{X}\times\mathcal{Y}$   if there exists $B>0$ such that 
\begin{enumerate}[label=(B\arabic*) , wide=0.5em,  leftmargin=*]
   \item   \label{cond:bernsteinF} for  all $f\in \mathcal{F}$, $  Pf^2 \le B Pf. $
  \end{enumerate}

Prior to stating our main result, we introduce  classes of functions that are constructed as convex combinations of  differences between functions in the original loss class and minimizers of the weighted risk $\risk_q$. For $q\in (0,1)$  recall that  $g_q^* $ minimizes  the $\risk_q$-risk over the class of score functions $\mathcal{G}$ and let $\mathcal{H}_q = \{ \ell_g - \ell_{g_q^*}, g\in\mathcal{G}\}$. Let
$$
{\mathcal{ H}} = \{ (1-q)h_qI_+ + q h_q I_-, \; q\in(0,1), h_q \in \mathcal{H}_q  \}. $$
With these notations notice already that for $h = (1-q)h_qI_+ + q h_q I_- \in {\mathcal{H}}$, with $ h_q = \ell_{g}-\ell_{g_q^*}$,   the quantity $P h$ should be interpreted as  an excess of weighted risk, since  
\begin{align*}
  \frac{1}{q(q-1) } P h & = \frac{1}{q} P(h_q I_+ ) + \frac{1}{1-q}P(h_q I_- )\\ 
                        &  =  P_q h_q = P_q( \ell_g - \ell_{g_q^*})\\
  &= \risk_q(g_q) - \risk_q(g_q^*).
\end{align*}
It turns out that the class $\mathcal{H}$ indeed satisfies a Bernstein assumption under standard assumptions  regarding the \emph{original} loss class  $\mathcal{L} = \{\ell_g,g\in\mathcal{G}\}$. Namely it is enough to assume that the latter satisfies a   strong convexity property and a Lipschitz property,  which are commonly satisfied in Machine Learning problems. 
  The proof of the following statement  is deferred to the Supplementary material
  \begin{lemma}[Sufficient conditions for $\mathcal{H}$ to satisfy a Bernstein-condition]\label{lemma:bernstein-cond}
    Assume that $\mathcal{G}$ is a convex subset of a normed vector
    space, and that there exists $L,\lambda>0$, such
    that for $s\in \{+, -\}$
    the functions $ g\mapsto P\ell_g I_s $ 
    are  
    respectively $(p\lambda)$-strongly convex and
     $((1-p)\lambda )$-strongly convex.
    Assume also that for
    $g_1,g_2\in\mathcal{G}$, 
    $\sqrt{ P( (\ell_{g_1} - \ell_{g_2})^2I_+ )} \le ( \sqrt{p}\, L )  \|g_1
    - g_2\|$ and
    $\sqrt{ P( (\ell_{g_1} - \ell_{g_2})^2I_- )} \le ( \sqrt{1-p}\, L)
    \|g_1 - g_2\|$.  Then ${ \mathcal{H}}$ satisfies the Bernstein
    condition~\ref{cond:bernsteinF},
    with $B = L^2/\lambda$.
  \end{lemma}
    \begin{example}\label{example:linear}
   Assume that the domain $\mathcal{X}$ is bounded in $\rset^d$ \ie, there exists
  some $\Delta_X>0$ such as
  $\forall x \in \mathcal{X}, \left\|x\right\|\leq \Delta_X $ for some 
   norm $ \|\cdot\|$. Consider the family of classifier and
  the loss function 
  $\mathcal{G}_u=\left\{g_\beta: x\mapsto \beta^Tx \, \,  \left\|\beta\right\|\leq
    u\right\}$ and $\ell_{g_\beta}(X,Y) =\phi(\beta^T XY)$, where $\phi: \rset \mapsto \rset $ is a  twice continuously
  differentiable non-decreasing function which is $\mu$-strongly convex 
  for some
  $\mu>0$. Then we have that 
  the domain $ I = \{ \beta^\top x y\,,, \|\beta\|\le u, x \in \Delta_x, y\in\{-1,+1\} \}$ is bounded 
  and the derivative $\phi'$ satisfies $\sup_{t\in I} |\phi'|(t) =D<\infty$. 
 Let    $V_s = \EE[XX^\top\,|\, Y= s.1], s = +, -$ be the second moment matrix of each class and $\sigma_{\max}^2, \sigma_{\min}^2$ their maximum and minimum eigenvalues. 
  Direct computations show that the Lipschitz and convexity
  constraints in Lemma~\ref{lemma:bernstein-cond} are respectively
  $L = D \sigma_{\max}$, 
  and
  $\lambda = \mu \sigma_{\min}^2$. 
  Under the condition that each class distribution is non
  degenerate, \ie does not concentrate on any lower dimensional
  subspace of $\rset^d$, we have $\sigma_{\min}>0$, thus
  $\lambda>0$. 
   The assumptions of Lemma~\ref{lemma:bernstein-cond} are
  thus satisfied, so that Bernstein condition~\ref{cond:bernsteinF} holds
  true with
  $$B = L^2/\lambda = \frac{D^2\sigma_{\max}^2 }{\mu\sigma_{\min}^2 }. $$
  \end{example}

Another useful property of ${ \mathcal{H}}$ is that  it inherits 
the regularity property of $\mathcal{L}$, see the Supplementary material for details. 
\begin{lemma}[VC-property of $\mathcal{H}$]\label{lemma:VC-tildeF}
  If $\mathcal{L} = \{\ell_g, g\in\mathcal{G}\}$ is of VC-type with
  envelope $U$ and parameters $(v,A)$ then ${\mathcal{ H }}$ also is VC with
  envelope $2U$ and parameters  $(\tilde v = 4v+1, \tilde A = 6A)$.  
\end{lemma}

We now state our main result (fast rates for the deviations of weighted risks)  in the light of the good properties of $\mathcal{H}$ stated below. To wit, in our application to empirical risk minimization (Corollaries~\ref{coro:fast-rates-excessrisk},~\ref{coro:fast-rates-contrained}),  the classes $\mathcal{F}_q$ in the following statement will be chosen  as $\mathcal{H}_q$, so that we shall have $\mathcal{F} = \mathcal{H}$. 

\begin{theorem}\label{theo:fast-rates}[Fast rates for the deviations of weighted probabilities]
  Let $(\mathcal{ F}_q, \, q \in (0,1))$ be a family of classes of
  functions 
  with common envelope $2U>0 $. Assume that the class of convex combinations 
  $ \{(1-q) f_q  I_+ + q f_q I_-,\;  q\in(0,1), f_q\in\mathcal{F}_q \} $
 satisfies Bernstein condition~\ref{cond:bernsteinF}  for some $B\ge 2U$, and that it  is of  VC-type with parameters $( \tilde v, \tilde A)$. 

 Then 
 the deviations of the  weighted probabilities $P_q$ over the classes $\mathcal{F}_q$ are uniformly controlled as follows:
for any $K>1$ and every $\delta >0$, with probability at least
 $1 - \delta $, for all $q \in (0,1)$  and 
 for all $f_q\in\mathcal{F}_q$, 
 \begin{align*}
   P_q (f_q)\leq \frac{K}{K-1}P_{n,q}(f_q) 
   + \frac{ c_1 B  K \tilde v \log(  5\tilde A  \sqrt{ n }  / \delta)}{
   2 nq(1-q)} , 
 \end{align*}
 where $c_1>0$ is an explicit universal constant given in the
 proof. \\Also, with probability at least $1 - \delta $,
 $ \forall q, \forall  f_q \in {\mathcal{F}}_q$,
 \begin{align*}
   P_{n,q} (f_q) 
   \leq \frac{K+1}{K}P_q(f_q) 
   + \frac{ c_1 B  K    \tilde v \log(  5\tilde A  \sqrt{ n }  / \delta )}{
   2 nq(1-q)} .  
 \end{align*}
\end{theorem}
\begin{proof}[Sketch of proof]
  
  The main tool for the proof is a fast rate result (Theorem 3.3  in \cite{Bartlett2005}, recalled for completeness in the supplementary material as  Theorem \ref{theo:fast-rates-ingredient}).
  The argument from the cited reference relies on a fixed point technique relative to a sub-root function upper bounding some local Rademacher complexity. Leveraging fine controls of the latters (Section~\ref{sec:fast-rates-proof}) we   establish that the  fixed point of the sub-root function is  of order $O(\log(n)/n)$ and we obtain an explicit control of  the deviations of the (standard) empirical measure, under a Bernstein condition (see Proposition~\ref{prop:simpleFastRateDeviation}). Finally the main result is obtained by applying the latter proposition to the specific class of convex combinations defined in the statement, and rescaling the obtained bound by the quantity $2q(1-q)$, see   Section~\ref{sec:fast-rates-proof} for details.   

\end{proof}
\begin{discussion*}
	 Similar proof techniques can be found in the standard classification literature, for example Corollary 3.7 in \cite{Bartlett2005}. Nevertheless, this particular work primarily concentrates on loss functions with binary values, namely $\left\{0,1\right\}$. The proof is based upon the fact that these functions are positive, and it employs the initial definition of the \VC dimension \citep{vapnik71}. In contrast, other existing works (e.g. Theorem 2.12 in \cite{bartlett2006empirical} or Example 7.2 in \cite{Gine2006}) demonstrate accelerated convergence rates for the \textit{typical} empirical risk minimizers, which do not extend to their balanced counterparts. The present result is more general, as it is uniformly applicable to a broader range of bounded functions and encompasses a more extensive definition of the \VC class. This notable extension facilitates the establishment of fast convergence rates for the excess risk of  (\ML) algorithms employed in imbalanced classification scenarios, such as cost-sensitive logistic regression and balanced boosting \citep{Menon2013,koyejo2014consistent,Tanha2020,xu2020class}.  In the remainder of this  section we provide examples of algorithms that verify the assumptions  of Theorem \ref{theo:fast-rates}.
\end{discussion*}

As an application of Theorem~\ref{theo:fast-rates}, we derive fast rates for the excess risk of empirical risk minimizers.
Our result (Corollary~\ref{coro:fast-rates-excessrisk} below) is  stated in terms  of excess of $\mathcal{R}_{\hat p}$ risk.  

\begin{corollary}\label{coro:fast-rates-excessrisk}
  Assume that $\mathcal{L}=\{\ell_g\, :\, g\in \mathcal G\}$ is of 
  VC-type with envelope $U>0$ and parameters $(v,A)$ and assume that
  ${\mathcal{H}}$ defined at the beginning of this section satisfies
  the Bernstein condition~\ref{cond:bernsteinF} 
  for some
  $B\geq 2 U $ (this is the case \eg under the Assumptions of
  Lemma~\ref{lemma:bernstein-cond}).  Let $\hat g_{\hat p}$ be a
  minimizer of the empirical balanced risk $\mathcal{R}_{n, \hat p}$
  considered in
  Section~\ref{sec:standard}.  
  Then for
  $\delta>0$, 
  with probability $1- \delta$, 
  \begin{align*}
& \mathcal{R}_{\hat p}(\hat g_{\hat p}) - \mathcal{R}_{\hat p}(g_{\hat p}^*) \le 
 \frac{ c_1 B      \tilde v \log(  5\tilde A  \sqrt{ n }  / \delta  )    }{2n \hat p(1-\hat p)},     \end{align*}

where $(\tilde v = 4v+1, \tilde A = 6A)$, and where the constant $c_1$
is the same as in Theorem~\ref{theo:fast-rates}.
\end{corollary}
\begin{proof}
We consider  the classes of functions  $\mathcal{F}_q = \mathcal{H}_q$. The  class of convex combinations from the statement of Theorem~\ref{theo:fast-rates} is precisely $\mathcal{H}$, which is $B$-Bernstein by assumption and it is also of VC-type with parameters $(\tilde v, \tilde A)$ by virtue of Lemma~\ref{lemma:VC-tildeF}. We may thus apply Theorem~\ref{theo:fast-rates}.  
  Because the  result of Theorem~\ref{theo:fast-rates} 
  holds uniformly over $q\in(0,1)$, $f\in\mathcal{F}_q$,  one may choose $q=\hat p$. Also we choose   $f_{\hat p}\in\mathcal{F}_{\hat p} = \mathcal{H}_{\hat p}$ as   $f_{\hat p} = \ell_{\hat g_{\hat p}} -  \ell_{g_{\hat p}^*} $. 
   Then the first term on the right-hand side of the first upper bound in Theorem~\ref{theo:fast-rates}  is nonpositive, and the result follows upon choosing $K=2$.
\end{proof}
\begin{remark}[Weighting with $\hat p$ or $p$]
Under mild conditions on $np$ and $\delta$, we have that $\hat p (1- \hat p) \geq p(1-p)/2  $ (as indicated by Chernoff’s multiplicative bound in Theorem \ref{th_chernoff}), so that Corollary~\ref{coro:fast-rates-excessrisk} immediately yields a rate of convergence in terms of the true value $p$ rather than its  empirical counter part. However whether it is possible to replace $\risk_{\hat p}$ with $\risk_{p}$ in the statement
remains an open question. The main bottleneck seems to be that replacing $\hat p$ with $p$ in expressions of the kind `$P_n( fI_+)/\hat p + P_n(f I_-)/(1-\hat p)$', where one term in the summand may be negative, induces additional slow rate terms of order $O(1/\sqrt{np})$. 
  \end{remark}

We conclude this section by  illustrating the significance of our results, through the concrete setting of Example~\ref{example:linear}. 
The following corollary is a direct consequence of Corollary \ref{coro:fast-rates-excessrisk} and guarantees fast rates of convergence for constrained \ERM, specifically for algorithms   of the form  $\hat g_{u,\hat p}(x)=\hat \beta_u^{T} x$ with
\begin{align*}
\hat \beta_u 
  =\argmin_{\left\|\beta\right\| \leq u}   n^{-1}\sum_{i=1}^{n}  &\phi(\beta^\top X_iY_i) \Big(p^{-1}\un{Y_i=1} +  \dotsb \\
& (1-p)^{-1}\un{Y_i =-1} \Big). 
\end{align*}
Then  from standard arguments the class $\mathcal{L} = \{\ell(x,y) = \phi(\beta^\top \}$ is of VC-type with parameters $(v = 2(d+1), A)$ for some $A>0$ (see e.g.\cite{van1996weak}, Chap.2.6) depending on $\phi$.  The following result derives immediately from the argument of Example~\ref{example:linear}, from  Lemma~\ref{lemma:bernstein-cond} and from Corollary~\ref{coro:fast-rates-excessrisk}. 
 
\begin{corollary}\label{coro:fast-rates-contrained} 
In the setting of Example~\ref{example:linear}, let $(v=2(d+1),A)$ be the parameters of the VC-type class $\mathcal{L}$ defined above the statement.   The excess risk of $\hat g_u$  verifies, for any $\delta>0$,
 with probability $1-4\delta$, 
 \begin{align*}
   \risk_{\hat p}\left(\hat g_{\hat p} \right)\leq \risk_{\hat p}( g^*_{\hat p})
+ \frac{ c_1 (d+1)D^2\sigma_{\max}^2 }{ \mu \sigma_{\min}^2} 
\frac{\log(30 A \sqrt{n}/\delta ) } { n \hat p (1-\hat p) }, 
 \end{align*}
 where $c_1$ is as in Theorem~\ref{theo:fast-rates}. 
\end{corollary}

\begin{discussion*}
In the context of constrained logistic regression, where $\phi(x) = \log(1 + e^{-x})$, the latter corollary yields fast convergence rates with constants and $L' = 1$, along with $\lambda = e^{-u}$. Corollary~\ref{coro:fast-rates-contrained} further establishes accelerated convergence rates for constrained empirical \emph{balanced} risk minimization with respect to losses such as mean squared error, squared hinge, and exponential loss, among others. This outcome aligns with expectations, as constrained empirical risk minimization is equivalent to penalization \citep{lee2006efficient,homrighausen2017risk}. Numerous studies have demonstrated the effectiveness of penalization in achieving rapid convergence rates \citep{Koren15,vanerven15a}. This aspect is particularly significant in the present context, as the standard convergence rate for imbalanced classification is $1/\sqrt{np}$, and accelerating the convergence rate leads to a more pronounced impact.
\end{discussion*}

\section{Numerical illustration}
In this section, we  illustrate on synthetic data 
our theoretical results on $k$-NN classification (Corollary \ref{coro:knn-consistency}) and on logistic regression (Corollary \ref{coro:fast-rates-contrained}). In both cases, particular attention is paid to highly imbalanced settings where $p=n^{-a}$ for some $0<a<1$.
Due to space constraint, the real data experiments are postponed to the supplement.
We use the following simple data generation process to obtain a   binary classification i.i.d  dataset $ (X_i , Y_i )_{i=1,\ldots , n} $,  with $X_i  \in
 \rset^2  $ and $Y \in \{-1, 1\}$. The $Y_i$ are Bernoulli variables with parameter  $ p = 1/n^a  $, for some $a<1 $. Then, given $Y_i=y$,  $X_i$ is drawn according to a $t$-multivariate-student distribution, with   parameters
$(\mu_y , \sigma_y , \nu_y )$, where  $(\mu_{-1}, g\mu_1 )= ((0,0),(1,1))$, $\sigma_1 = 3 \sigma_{-1}=3 I$ and $(\nu_{-1},\nu_1) = (2.5,\,1.1)$.
\subsection{Balanced $k$-Nearest Neighbors}
Corollary \ref{coro:knn-consistency} provides sufficient conditions on $k,n$ and $p$ for consistency of the  $k$-NN classification rule, the key being that 
$kp$ should go to $\infty$. This  suggests the existence of a learning frontier on the set $(k,p)$ above which consistent learning is ensured. Our experiments aim at illustrating this fact.
%
Here the training size is $n=1e4$. We set $p=p_n = 1/n^a$ and $k=n^b$, where  $a,b$ vary within  the interval $[1/4,3/4]$ and cover different cases ranging from $pn \to  0 $ to $pn\to \infty$. The \textrm{AM}-risk for the classification error associated to the balanced $k$-NN classifier (estimated with $20$ simulations) is displayed as a function of $(k,p)$ in Figure~\ref{fig:knn-heatmap}. For small values of $kp$, the performance of the $k$-nearest neighbors classifier mirrors that of a random guess, maintaining an \textrm{AM} risk near $0.5$, while $kp$ large ensures good performance.
This observation illustrates (and extends) the conclusion of Corollary \ref{coro:knn-consistency}, supporting that consistency is obtained if (and only if) $kp\to \infty$. 

%
%


\begin{figure}
	\includegraphics[width=\linewidth]{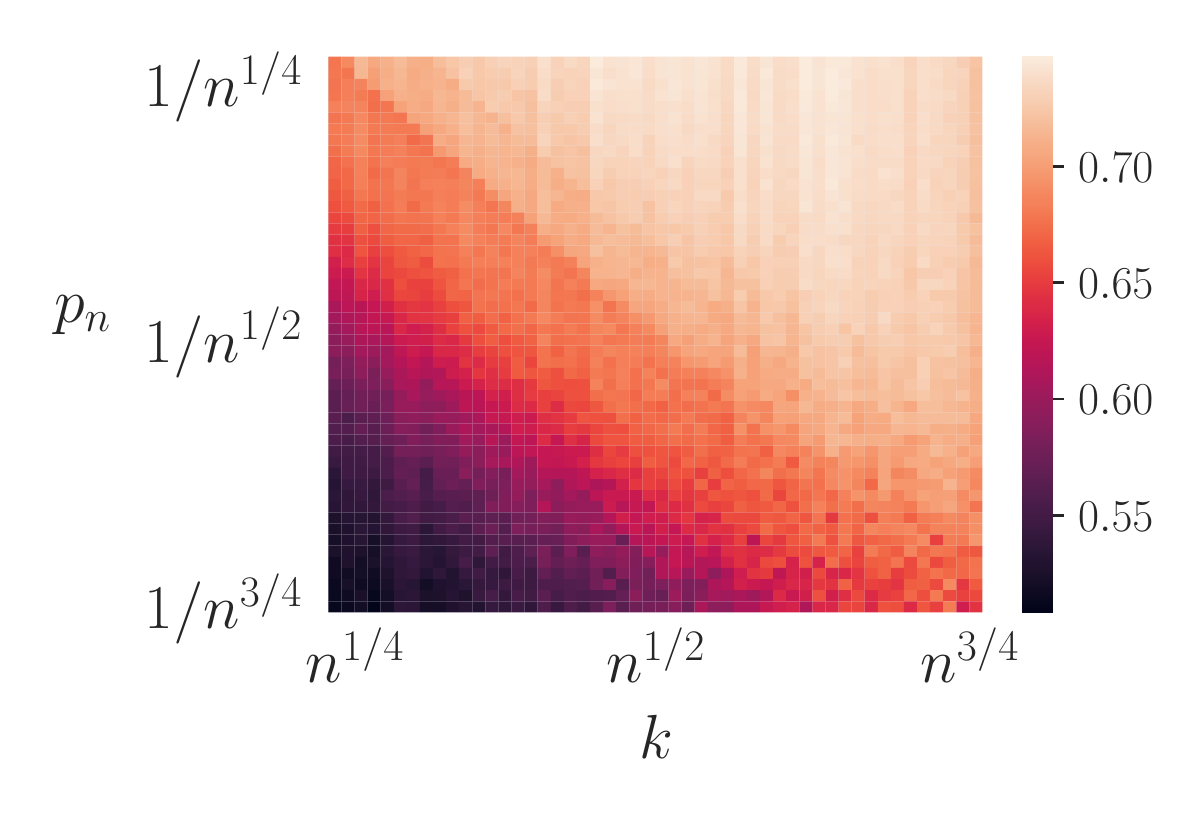}
	\caption{\textrm{AM} risk of the balanced $k$-NN (heatmap).}\label{fig:knn-heatmap}
\end{figure}
%
%
%
\subsection{Balanced ERM}
Our goal is to demonstrate empirically  that  the fast convergence rate of order $1/ (np) $ obtained in Corollary \ref{coro:fast-rates-contrained} is  sharp and can be observed in practice for a wide range of values of $p$. 
We consider the linear setting of Example~\ref{example:linear}, 
with the logistic loss: $\ell_g( X, Y ) =~\log(1-e^{-g(X)Y})$, $g(X) = \beta^T X$ and $\|\beta\|\le u=10$. The sample size $n $  ranges over the grid $[100,1e4]$ and we let  $p = n^{-a}$,  $ a \in \{1/3,1/2,2/3 \}$. 
Some Monte-Carlo simulations ($N=1e{5}$ simulations) are needed to estimate  $g^*_{\hat p}$. For simplicity and to alleviate the computational burde  we consider  $g^*_{p} $ for fixed $p$ instead,   and the balanced risk $\mathcal{R}_p$ instead of $\mathcal{R}_{\hat p}$. 
The risk function $\risk_{p}$ is also estimated based on Monte-Carlo simulations $(N' = 1e4$). We thus obtain both $\risk_{p}(g^*_{p}) $ and $\risk_{p}( \hat g_{\hat p} ) $, and the value of the excess $p$-risk follows  
We perform  report the average and
the upper $0.10$ and $0.90$ quantile of the absolute error obtained over the $n_{simu}=1e4$ experiments.  
\begin{figure}[!h]
	\centering
	\begin{subfigure}{0.4\textwidth}
		\includegraphics[width=\textwidth]{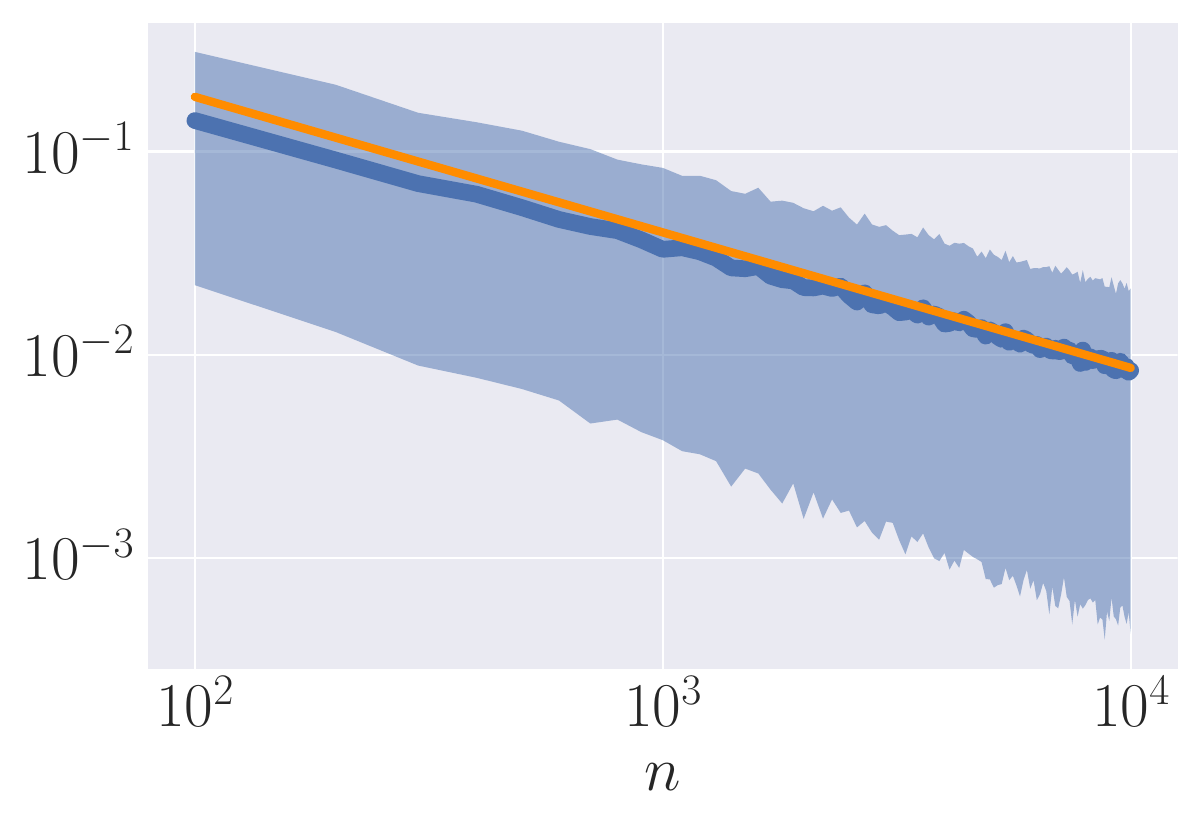}
		\caption{$p_n=1/n^{1/3}.$}
		\label{fig:first}
	\end{subfigure}
	\hfill
	\begin{subfigure}{0.4\textwidth}
		\includegraphics[width=\textwidth]{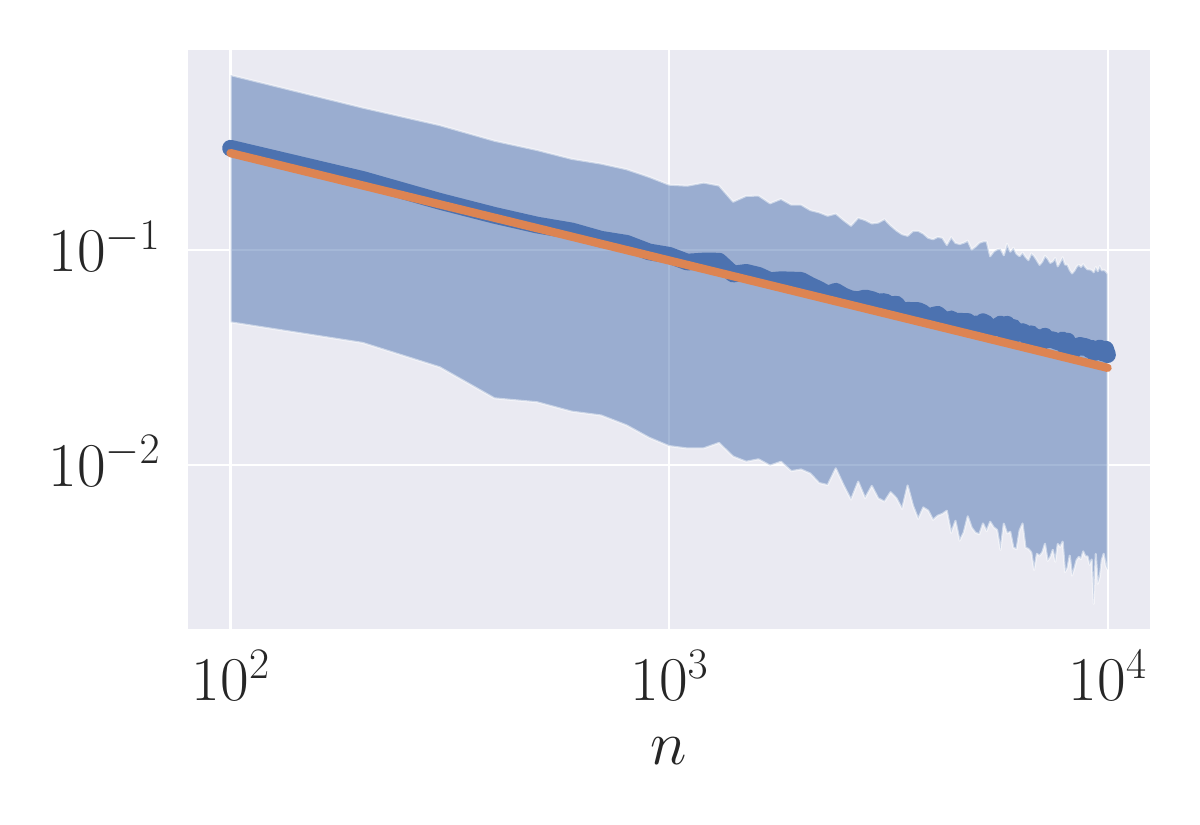}
		\caption{$p_n=1/n^{1/2}.$}
		\label{fig:second}
	\end{subfigure}
	\hfill
	\begin{subfigure}{0.4\textwidth}
		\includegraphics[width=\textwidth]{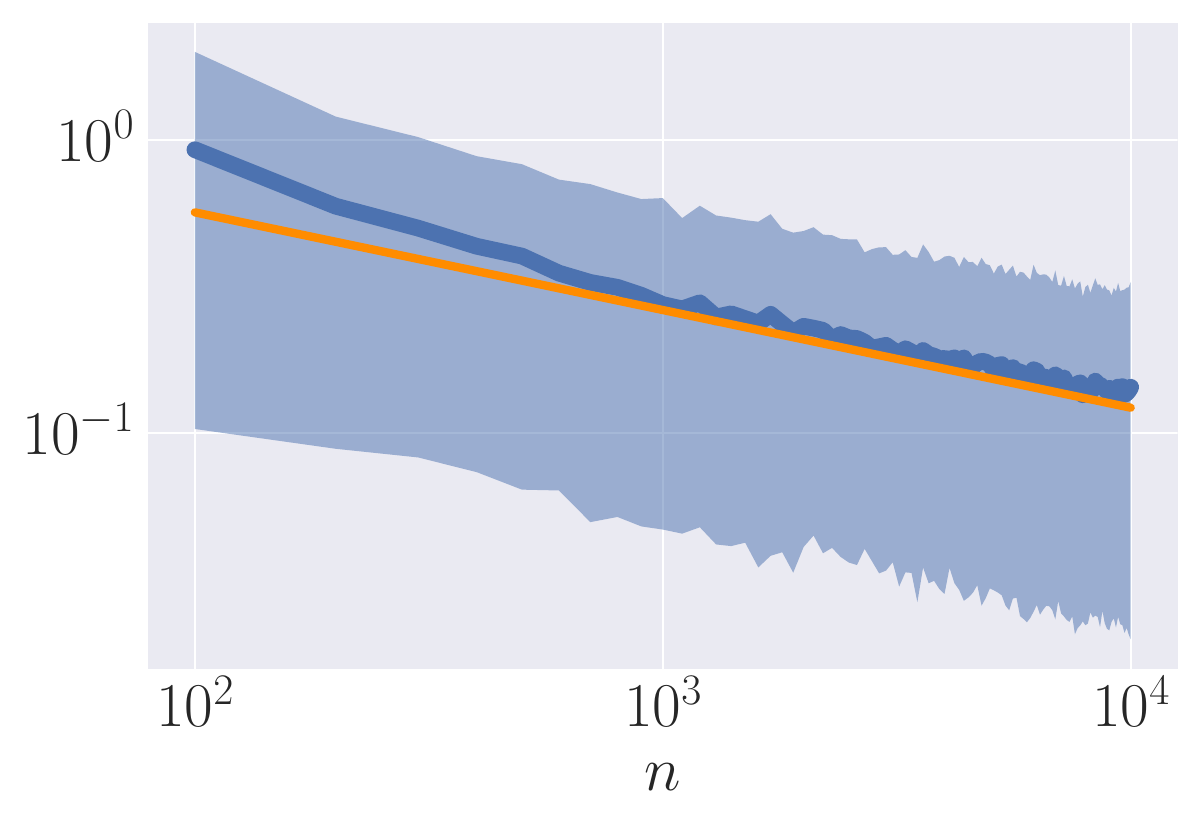}
		\caption{$p_n=1/n^{2/3}.$}
		\label{fig:third}
	\end{subfigure}
	\caption{Excess balanced risk (log-scale)  of logistic regression as a function of $n$, when $p=p_n\to 0$.  Orange line: curve $1/np$. Blue area: inter-quantile range $[0.1,0.9]$. }\label{fig:excessrisk}.
	\label{fig:figures}
      \end{figure}
Figure \ref{fig:excessrisk} displays the excess risk as a function of
the sample size
$n$ on a logarithmic scale, for
$a\in\left\{1/3,1/2,2/3\right\}$. Other figures exploring other
  values of
  $a$ are reported in the supplementary material. We notice that the
  excess risk vanishes in the same way as the function $ n\mapsto
  1/np$ confirming the accuracy of the upper bound from Corollary
  \ref{coro:fast-rates-contrained}.

\section{Conclusion}
In this paper, we have derived upper bounds for the balanced risk in highly imbalanced classification scenarios. Notably, our bounds remain consistent even under severe class imbalance ($p\to 0$), setting our work apart from existing studies in imbalanced classification \citep{Menon2013,koyejo2014consistent,xu2020class}. Furthermore, it is worth to highlight that this is the first study to achieve fast rates in imbalanced classification, marking a significant advancement in the field.
Our findings confirm  that both risk-balancing approaches and cost-sensitive learning are consistent across nearly all imbalanced classification scenarios. This aligns with experimental works previously documented in the literature \citep{elkan2001foundations,wang16,khan2018,wang19,PATHAK2022}. 
Furthermore, the methodologies and proof techniques presented in this paper are adaptable to other imbalanced classification metrics beyond balanced classification. Potential extensions include demonstrating consistency for metrics such as the $F_1$-measure, Recall, and their respective variants.



\bibliographystyle{apalike}
\bibliography{biblio_imbalanced}
\newpage
\appendix
\onecolumn
\section{Appendix}

\section{Auxiliary results}

The following standard Chernoff inequality is stated and proven in  \cite{hagerup1990guided}.

\begin{theorem}\label{th_chernoff}
	Let $(Z_i)_{i\geq 1}$ be a sequence of i.i.d. random variables valued in $\{0,1\}$. Set $\mu =  n P (Z_1)$ and $S = \sum_{i=1} ^n Z_i$.   For any $\delta \in (0,1)$ and all $n\geq 1$, we have with probability $1-\delta$:
	\begin{align*}
	S \geq \left(1- \sqrt{ \frac{2 \log(1/\delta)  }{  \mu} } \right) \mu  .
	\end{align*}
	In addition, for any $\delta \in (0,1)$ and $n\geq 1$, we have with probability $1-\delta$:
	\begin{align*}
	S \leq \left(1 +  \sqrt{ \frac{3 \log(1/\delta)   }{  \mu} }  \right) \mu  .
	\end{align*}
\end{theorem}

The following is taken from \cite{portier2021nearest}. Other similar results are given in \cite{gine2001consistency} or \cite{plassier2023risk} with non-explicit constants. 

\begin{theorem}\label{th_vc_class}
Let $(Z_1,\ldots ,Z_n) $ be an independent and identically distributed collection of random variables  with common distribution $P$ on $ (S,\mathcal S)$. Let $\mathcal G$ be of VC-type with parameters $v\geq 1 $, $A\geq 1$ and uniform envelope $U\geq \sup_{g\in \mathcal G,\, x\in S}  |g(x)|$. Let $\sigma$ be such that $\sup_{g\in \mathcal G} \Var_P(g) \leq \sigma^2  \leq U^2 $.  For any $n\geq 1$ and $\delta\in (0,1) $, it holds, with probability at least $1-\delta$,
\begin{align*}
&  \sup_{g\in \mathcal G} \left| \sum_{i=1} ^n   \{g (Z_i )  -  P (g )  \} \right|  \leq  K_1 
 \sqrt{  vn \sigma^2   \log\left(  9AU  /  (\sigma \delta)    \right)      }  +   K_2   Uv   \log\left(  9 AU  /  (\sigma \delta)     \right)  ,
  \end{align*} 
	with $K_1 =    5C  $,  $K_2 =   64 C^2    $ and  $C=  12 $. 
\end{theorem}

\begin{lemma}\label{lemma_vc_pres}
 Suppose that $\mathcal F $ (resp. $\mathcal G$) defined on $(S,\mathcal S)$ is of VC-type with envelope $U$ and parameter $(v,A)$ and let $E\in  \mathcal S$. The following holds: 
\begin{enumerate}
\item $\{ f I_E  \, :\, f \in \mathcal F\} $ is of VC-type with envelope $U$ and parameter $(v,A)$,
 \item $\mathcal{F} - \mathcal{G} = \{f - g  \, :\, f \in \mathcal F, \,  g \in \mathcal G\}$ is  of VC-type with envelope $2U$ and parameter $(2v,2A)$,
\item  $\{ f - P(f|E  ) \, :\, f \in \mathcal F\} $ is of VC-type with envelope $2U$ and parameter $(2v,A)$,
\item  $\{ qf +  (1-q) g  \, :\, f \in \mathcal F , \,  g \in \mathcal G, \, q\in [0,1] \} $ is of VC-type with envelope $U$ and parameter $(2v+1,3A)$.
\end{enumerate}
\end{lemma} 	

\begin{proof}

  Let $Q$ be a probability measure on $(S,\mathcal S)$. Let
  $(f_k)_{k=1,\ldots, K} $ be the center of an $\epsilon U$-covering
  of $(\mathcal F,Q)$. The first 
  statement follow from the fact that
  $\| f1_E - f_k1_E \|_{L_2(Q)}\leq \| f - f_k \|_{L_2(Q)}$.
For the second statement consider $U\epsilon$-covers $(f_k, k\le K)$ and  $(g_j, j\le J)$ respectively of $\mathcal{F}$ and $\mathcal{G}$. Then the triangle inequality shows that $(f_k - g_j), k\le K,j\le J$ forms a $2U\epsilon$-cover of $\mathcal{F}+\mathcal{G}$ and the result follows. 
  Now let
  $(\tilde f_k)_{k=1,\ldots,K} $ be the center of an
  $\epsilon U$-covering of $(\mathcal F,P_E)$ with
  $P_E(\cdot) = P (\cdot |E ) $. Consider the covering induces by the
  centers $(f_k - P_E (\tilde f_j))_{1\leq k,j \leq K} $ made of $K^2$
  elements. Suppose that $f\in \mathcal F$. Then there is $k$ and $j$
  such that
\begin{align*}
 \| (f - P_E (f) ) - (f_k - P_E (\tilde f_j) \|_{L_2(Q)} &\leq \| f -  f_k  \|_{L_2(Q)} +  P_E (f - \tilde f_j) \\
 & \leq  \| f -  f_k  \|_{L_2(Q)} +  \| f - \tilde f_j \| _{L_2(P)}\\
 & \leq 2U \epsilon.
\end{align*}
 Hence we have found a $2U \epsilon$-covering of size $K^2$ which by assumption is smaller than $( A/\epsilon)^{2v}$. This implies the third statement of the lemma. For the last statement, 
 let $\mathcal H = \{ qf +  (1-q) g  \, :\, f \in \mathcal F , \,  g \in \mathcal G, \, q\in [0,1] \}  $. 
Let $(f_k)_{k=1,\ldots, K} $ (resp. $g_\ell$, $(g_\ell)_{\ell = 1,\ldots, L}$)   be the center of an $\epsilon U$-covering of $(\mathcal F,Q)$ (resp. $(\mathcal G,Q)$). Let $q_i$, $i = 1,\ldots, \lfloor 1/\epsilon \rfloor $ be an $\epsilon $-covering of $[0,1]$. 
 Let $h = qf + (1-q) g $ be such that $q\in [0,1]$, $f\in \mathcal F$ and $g\in \mathcal G$. There is $f_k,g_\ell, q_i$ such that
\begin{align*}
&\| qf +  (1-q) g    -   (q_i f_k +  (1-q_i) g_\ell )  \|   _{L_2(Q)}\\
& \leq  \|   q (f - f_k) + (1- q)  (g- g_\ell ) \|   _{L_2(Q)}   + \|  (q -q_i) f_k + (q_i - q)  g_\ell    \|   _{L_2(Q)}  \\
&\leq q U \epsilon + (1-q) U \epsilon + \epsilon U  + \epsilon U = 3 \epsilon U .
\end{align*}  
Hence the element $(q_i f_k +  (1-q_i) g_\ell )$ form an $3 \epsilon U$-covering in the space $L_2(Q)$. There are $ KL \lceil 1/\epsilon \rceil\leq  KL /\epsilon  $ such elements. As a consequence, since  $\mathcal F $ and  $\mathcal G$ are are of VC-type, it follows that
$$ \mathcal N ( \mathcal H, L_2 ( Q ) , 3 \epsilon U  ) \leq ( A / \epsilon ) ^{ 2 v} (1/ \epsilon) \leq ( A/\epsilon)^{2v+1}$$
  which implies the stated result.

\end{proof}

The next lemma generalizes Theorem 17.1 from \cite{biau2015lectures} to the balanced type classifiers.

\begin{lemma}\label{classification_basic_bound}
For any classifier $g $ that writes $g(x) = \sign( \nu(x) - 1 ) $, $x\in \mathcal X$, we have
$$\risk_{p}  (g) - \risk_{p}  (g^*_{p}) =\EE\left[\ind_{ g(X) \neq g^*_p(X)}\frac{\lvert \eta (X) -p \rvert}{p(1-p)}\right],$$
 where $g^*_{p}$ is the balanced Bayes classifier (introduced in Section \ref{sec:background}). Furthermore, whenever $ p\leq 1/2$,
$$ \risk_{p}  (g) - \risk_{p}  (g^*_{p}) \leq 2 \EE \left[  \left|   \nu(X)  - \nu^*(X)\right|\right]$$
where $\nu^* (x)  = { \eta (x)} / { p  } $.
\end{lemma}

\begin{proof}
	The balanced risk writes as
	\begin{align*}
	\risk_{p}  (g)&=P_+ \left(\nu(X)< 1 \right)+ P_-\left(\nu(X)\geq 1 \right)\\
	& = \EE\left[\frac{\mathbb I_{(\nu(X)< 1)} \mathbb I_{Y = 1 }}{p}+ \frac{\mathbb I_{(\nu(X)\geq 1)} \mathbb I_{Y = -1 }}{1-p}\right].
	\end{align*}
	In addition, using a conditioning argument yields,
	$$\risk_{p}  (g) = \EE\left[\frac{\mathbb I_{(\nu(X)< 1)} \eta(X)}{p}+ \frac{\mathbb I_{(\nu(X)\geq 1)} (1-\eta(X))}{1-p}\right].$$
		Similarly we have
	$$\risk_{p}  (g^*) = \EE\left[\frac{\mathbb I_{(\nu^*(X)< 1)} \eta (X)}{p}+ \frac{\ind_{(\nu^*(X)\geq 1)} (1-\eta(X))}{1-p}\right].$$
	It follows that
	\begin{align*}
	\risk_{p}  (g) - \risk_{p}  (g^*_p) &=\EE\left[\mathbb I_{\sign( \nu^*(X)-1)\neq \sign(\nu(X)-1)}\frac{\lvert \eta(X) -p \rvert}{p(1-p)}\right]\\
	&=\EE\left[\mathbb I_{g^*(X)\neq g(X)} \frac{\lvert \eta (X) -p \rvert}{p(1-p)}  \right],
	\end{align*}
		This concludes the first part. For the second part, it remains to note that for any real numbers $(x,y)$
	$$\sign( x -1)\neq \sign(y -1) \implies \lvert y -1 \rvert \leq \lvert x- y\rvert, $$
	so that, using that $\nu^*=\eta^*/p$, we obtain
	\begin{align*}
	\risk_{p}  (g) - \risk_{p}  (g^*) &=\EE\left[\mathbb I_{\sign( \nu^*(X)-1)\neq \sign(\nu(X)-1)}\frac{\lvert \eta (X) -p \rvert}{p(1-p)}\right]\\
	& = \EE\left[\mathbb I_{\sign( \nu^*(X)-1)\neq \sign(\nu(X)-1)}\frac{\lvert \nu^*(X) -1 \rvert}{(1- p) }\right]\\
&	\leq  \frac{\EE\left[\lvert\nu^*(X)- \nu(X)\rvert\right]}{1-p},
	\end{align*}
but 	since $p \leq 1/2$ we obtain the desired result.

\end{proof}

\section{Standard rates proofs}

\subsection{Proof of Theorem \ref{theo:VC-standard-rate}}

 We start with the following lemma which is a simple consequence of Theorem \ref{th_chernoff}. 

\begin{lemma}\label{lemma_for_p}
 Let $z_n = \sqrt{ 2\log(1/\delta)  / (n p) } $ and suppose that $z_n\leq 1 $. Then, with probability at least $1-\delta$, we have
\begin{align}\label{eq:p_n}
\frac p {\hat p} - 1 \leq     \frac{z_n}{ 1 - z_n}   
\end{align}
and whenever  $z_n\leq 1/2  $ we obtain that $ p / {\hat p} - 1 \leq 2 z_n\leq 1$, with probability greater than $1-\delta$.
\end{lemma}
	We have that
 	\begin{equation}
 	P_{n,+}\left(f \right) - P_+\left(f \right) = \frac{P_{n} \left( \left(f - P_+\left(f\right)\right)\ind_{\left\{Y=1\right\}}\right) } {\hat p}
 	\end{equation}
 	we focus on each term, denominator and numerator, separately. For the numerator, the term $\left(f - P_+(f)\right) \ind_{\left\{Y=1\right\}}$ has mean $0$. In virtue of Lemma \ref{lemma_vc_pres}, the class $ (f - P_+(f)) \ind_{\left\{Y=1\right\}}$ is still bounded by $2U$ and is still of VC-type with  VC parameter $(2v,A)$. 
As a consequence, we can use Theorem~2 in~\cite{portier2021nearest}, stated as Theorem \ref{th_vc_class} in the present supplementary file. The variance is bounded as follows
 	$$\Var\left(\left(f - P_+(f)\right) \ind_{\left\{Y=1\right\}} \right) =  P\left((f - P_+(f))^2 \ind_{\left\{Y=1\right\}} \right) = \Var_+(f  ) p= \sigma_{+} ^2   p,$$
 	by definition of $ \sigma_{+} ^2 $. As a consequence, Theorem \ref{th_vc_class} gives that
 	\begin{align*}
  \sup_{f\in \mathcal F} 	|P_n  \left( \left(f - P_+(f)\right) \ind_{\left\{Y=1\right\}}\right) | &\leq 	K' \left(   \sqrt{\frac{ v  \sigma_{+} ^2 p}{n} \log\left( \frac{ K' A U } {  \delta  \sigma_{+}  \sqrt{p}} \right)  }  + \frac{U v}{n} \log\left( \frac{ K' A U } {  \delta \sigma_+\sqrt{p}} \right) \right) \\
& \leq 2 K'  \sqrt{\frac{ v  \sigma_{+} ^2 p}{n} \log\left( \frac{ K' A U } {  \delta \sigma_+ \sqrt{p}} \right)  },
 	\end{align*}
 	where the last inequality has been obtained using the stated condition on $n$ and $\delta$.
  	For the denominator, using Lemma  \ref{lemma_for_p},  we have that, with probability $1-\delta$, $p \leq 2 \hat p$, by using the condition on $n$ and $\delta$. Using the union bound, we get, with probability $1-2\delta$,
 	\begin{align*}
 \sup_{f\in \mathcal F} 	\frac{P_n \left( \left(f - P_+(f)\right) \ind_{\left\{Y=1\right\}}\right)}{\hat p} \leq  
 	 4 K'  \sqrt{\frac{ v  \sigma_{+} ^2 }{n p } \log\left( \frac{ K' A U } {  \delta\sigma_+ \sqrt{p}} \right)  }
 	\end{align*}
	and the proof is complete.
\subsection{Proof of Corollary \ref{prop_min_risk}}

First, using the definition of $\hat g_{\hat p}$ yields
$$ \risk_{n,\hat p}\left(\hat g_{\hat p}\right)-\risk_{n,\hat p}\left(g^*_{ p }\right)\leq 0,$$	
so that	
\begin{align*}
  \risk_{p}(\hat g_{\hat p} ) -  \risk_p(g^*_{p})
  &\leq \risk_{p}(\hat g_{\hat p}  )-
    \risk_{n,\hat p}(\hat g_{\hat p} ) -
    \left(\risk_{p}(g^*_{p})-\risk_{n, \hat p}( g^{*}_{p} )
    \right)\\
  &\leq 2 \sup_{g\in \mathcal{G}}\left\lvert \risk_{p}(g)-\risk_{n,\hat p}( g)\right\rvert\\
	& \leq  \sup_{g\in \mathcal G} \left| P_{n,-} (g) - P_{-}(g)\right|+ \sup_{g\in \mathcal G} \left| P_{n,+} (g) - P_+(g)\right|.
	\end{align*}
	It remains to use Theorem \ref{theo:VC-standard-rate} twice, one time with $Y=1$ (as stated) and one more time with $Y=-1$. The end of the proof consists in verifying that the stated bound is an upper bound for each of the two previously obtained upper bounds.

\subsection{Proof of Theorem \ref{prop:as_knn_unif}}\label{sec:knn_consistent}

Even though the results are different, the proof is inspired from the ones of Theorem 1 in \cite{ausset2021nearest} and Theorem 6 in \cite{portier2021nearest}. First we recall two results that will be useful in the proof. In each, we assume that  \ref{cond:reg0} \ref{cond:reg1} and \ref{cond:reg2} are fulfilled.
The following Lemma~\citep[Lemma 4]{portier2021nearest} controls the size of the $k$-NN balls uniformly over all $x\in S_X$.

\begin{lemma}[{\citet[Lemma 4]{portier2021nearest}}]\label{prop:tau}
For all $n\geq 1$, $\delta \in (0,1)$ and $1\leq k\leq n$ such that $8 d \log(12n / \delta ) \leq k   \leq T ^d  n b_X c V_d /2 $, it holds, with probability at least $1-\delta$:
\begin{align}\label{eq:tau}
 \sup _{x\in S_X} \hat  \tau_{x}   \leq \overline \tau_{n,k} : =  \left(\frac{ 2  k }{ n b_X c V_d}  \right)^{1/ d},
\end{align}
 where $V_d = \lambda (B(0,1))$.

\end{lemma}

The next lemma is a consequence of Theorem \ref{th_vc_class}. Let $$ \mathcal G = \{  g(Y , X) = (\ind_{Y = 1} - \eta(X) )\mathbb I _ { \|X - x\|\leq \tau } : \tau\leq \overline \tau_{n,k}  ,\, x\in \mathbb R^d  \}$$
	which is of VC-type as shown in Lemma 9 in  \citet{portier2021nearest} (see also \cite{wenocur1981some}). 
Because $S_X $ is compact and $ \eta / p$ continuous, there exists $C$ such that $\eta(x) \leq p C$ for all $x\in S_X$. The variance of each element in the class is bounded as
	\begin{align*}
	\Var(g  (Y,X)) \leq E ( \ind_{Y = 1}  \mathbb I _ { \|X - x\|\leq \tau } ) 
 \leq  \int \eta(z)  \mathbb I _ { \|z - x\|\leq \tau } f_X(z) dz 	\leq   C p  U_X  \overline \tau_{n,k} ^d V_d .
	\end{align*}
	Injecting the previous variance bound (which scales as  $     p  k /  n $)  in the upper-bound  given in Theorem \ref{th_vc_class} we obtain the following statement.

\begin{lemma}
 We have with probability at least $1-\delta$,
\begin{align}\label{eq:supG}
  \sup_{ g\in \mathcal G} \left|  \sum_{i =1}^n g  (Y_i,X_i)  \right| \leq  C_1 \left( \sqrt{ k p \log\left( \frac{C_2 n}{  p \delta }  \right) }  + \log\left( \frac{C_2 n}{  p \delta }  \right)\right)   
\end{align}
where $C_1$ and $C_2$ are constants that does not depend on $n$, $k$ and $p$ but on the dimension $d$, the VC parameter of $\mathcal G$, and the probability measure $P_X$.
\end{lemma}

Define the event $E_n$ as the union of  \eqref{eq:p_n}, \eqref{eq:tau} and \eqref{eq:supG}. By the previous two lemmas and Lemma \ref{lemma_for_p}, using the union bound, we obtain that $P(E_n ) \geq 1-3\delta$. In light of Borel-Cantelli Lemma we choose $\delta = 1/n^2$ so that $\sum_n (1- P(E_n )) $ is finite and the event $\liminf_n E_n$ has probability $1$. It then suffices to show that $ E_n $ implies that $ { \hat \eta (x) } / {  \hat p}  -\nu^*(x)  = O  ( \sqrt{\log(n) / kp} + (k/n) ^{1/d}  )$. Note that under $E_n$, when $n$ is large enough, by  Lemma \ref{lemma_for_p}, $p / \hat p \leq 2$. Let $M_i = \ind_ { Y_i = 1   } - \eta (X_i) $ and $B_i(x) = \eta (X_i) -  \eta (x)$. We have
	\begin{align}\label{decomp_knn}
		\frac{ \hat \eta (x) } {  \hat p}  -\nu^*(x) 
		=
		&    \frac{  \sum_{i \in \hat I(x) }  M_i }{ k\hat p  }  +
 \frac{  \sum_{i \in \hat I(x)  }   B_i(x)   }{ k \hat p  } 		
		 + \eta(x) \left(\frac{1}{  \hat p } - \frac{1}{  p } \right).
	\end{align}
	 On the event $E_n$,	the function $ (Y,X) \mapsto (\ind_{Y = 1} - \eta (X) )  \ind_{ \|X - x\|\leq \hat {  \tau}_{x} } $ belongs to the space  $\mathcal G$. Consequently, 
	 	the first term in \eqref{decomp_knn} is smaller than 
	$$ (k\hat p)^{-1} \sup_{ g\in \mathcal G} \left|  \sum_{i =1}^n g  (Y_i,X_i)  \right| $$
	which by Lemma \ref{lemma_for_p} and \eqref{eq:supG} is $O  (   \sqrt { \log(n) / k p}  )$.
	Using the assumption that $ x\mapsto \eta (x)  / p $ is $L$-Lipschitz we get that, on $E_n$, the second term in \eqref{decomp_knn} is such that
\begin{align*}
 \frac{  \sum_{i \in \hat I(x)  }   B_i(x)   }{ k \hat p  } 	 \leq \frac{p }{\hat p } L  \overline \tau_{n,k} ,
\end{align*}
which, using Lemma \ref{lemma_for_p}, is  $ O ( (k/n)^{1/d}) $.
	The third term  in \eqref{decomp_knn}  is smaller than
	$  \left(\eta (  x) /  p \right) \left(  p /  \hat p  -  1 \right)$
	which is, using again the Lipschitz assumption and Lemma \ref{lemma_for_p} again, $ O ( \sqrt { \log(n ) / ( n  p )   }$. The latter bound is smaller than $   \sqrt { \log(n) / ( k p) } $ so it does not appear in the stated bound.

\section{Fast rates proofs}\label{sec:fast-rates-proof}

\subsection{Intermediate results}\label{sec:fast-rates-proof}

Before moving to the main proof we remind some necessary concepts and provide two technical lemmas inspired  from several papers dealing with empirical processes on VC-type classes  
 \cite{gine2001consistency,gine2009exponential}. First, let us recall the definition of a \emph{sub-root} functions.
\begin{definition}
	A function $\psi:[0, \infty) \rightarrow[0, \infty)$ is sub-root if it is nonnegative, nondecreasing and if $r \mapsto \psi(r) / \sqrt{r}$ is nonincreasing for $r>0$.
\end{definition}

Let $\sigma_1,\ldots, \sigma_n$ denote a collection of independent Rademacher variables, i.e., for each $i$, $\sigma_i\in\left\{-1,1\right\}$ and $P(\sigma_i=  1)= {1}/{2}$. The Rademacher variables are independent from the collection $Z_1 = (X_1,Y_1),\ldots, Z_n = (X_n,Y_n)$. In the sequel,  we will focus on the so called Rademacher complexity of functional classes  $\mathcal{F}$, defined as
$$R_n\left(\mathcal{F}\right)  =\frac{1}{n} \sup_{f\in \mathcal F}\sum_{i=1}^{n}\sigma_if(Z_i).$$
Namely a central object in  our proof will be  the  Rademacher complexity of classes $\mathcal{F}_q$ (differences between a given loss function and an optimal one) introduced at the beginning of Section~\ref{sec:fast}.
The conditional expectation given $(Z_i)_{1\leq i\leq n}$ (taken with respect to the Rademacher variables $(\sigma_i)_{1\leq i\leq n}$ only), will be denoted $E_\sigma$. 
Define
\begin{align*}
\mathcal F_{n,r} = \{f\in  \mathcal F\,:\, P_n(f^2) \leq r\} .
\end{align*} 
When $r> U^2$, since $P_n(f^2) \leq U^2$, we have that $\mathcal F_{n,r} = \mathcal F_{n,U^2}$. Therefore we assume subsequently that $   r  \leq U ^2   $.
In the next lemma we derive an upper bound for  $ E_\sigma R_n (\mathcal{F}_r)$.
\begin{lemma}\label{lemma:sub-root-UB}
	Let $\mathcal{F}$ be a class of functions that is VC-type with envelope $U>0$ and parameter $v,A\geq 1$. For any $r\leq U^2$, it holds that, with probability $1$,
\begin{align*}
 E_\sigma R_n( \mathcal F_{n,r}  ) \leq  C  \sqrt{ rn^{-1} v  \log(  eA U     /  \sqrt r )  }      .
\end{align*}
with $C = 12 \int_{  1 } ^\infty   s^{ - 2}  \sqrt{  1 + \log(  s  )   }     d s$. 
\end{lemma}

\begin{proof}
	Using  Dudley's entropy integral bound, see for instance  Corollary 5.25 in \cite{van2014probability}, one has,
	$$E_\sigma\left[R_n(\mathcal{F}_{n,r})\right]\leq \frac{12 }{\sqrt{n}}\int_{0}^{\infty} \sqrt{\log\mathcal N \left(\mathcal{F}_{n,r},  L_2(P_n)  ,  \epsilon  \right)}d\epsilon $$
	 By definition of  $\mathcal{F}_{n,r}$, it holds that $\mathcal N \left(\mathcal{F}_{n,r},  L_2(P_n)  ,  \epsilon  \right)=1$  as soon as $\epsilon \geq \sqrt{r}$. Hence, using some variable changes, we obtain
	\begin{align*}
			E_\sigma\left[R_n(\mathcal{F}_{n,r})\right]& \leq \frac{12 }{\sqrt{n}}\int_{0}^{\sqrt r } \sqrt{\log\mathcal N \left(\mathcal{F}_{n,r},  L_2(P_n)  ,  \epsilon  \right)}d\epsilon 	\\
			& =\frac{12 U }{\sqrt{n}}\int_{0}^{\sqrt r /U  } \sqrt{\log\mathcal N \left(\mathcal{F}_{n,r},  L_2(P_n)  ,  U \epsilon  \right)}d\epsilon 	\\
			& \leq \frac{12 U \sqrt{v} }{\sqrt{n}}\int_{0}^{\sqrt r /U  } \sqrt{\log ( A/\epsilon ) }d\epsilon \\
			& = \frac{12 U  A \sqrt{v} }{\sqrt{n}}\int_ { A U / \sqrt r   } ^\infty \sqrt{\log (s ) } s^{-2} ds \\
						& = \frac{12   \sqrt{vr } }{\sqrt{n}} \int_ { 1 } ^\infty \sqrt{\log (s A U / \sqrt r   ) }   s ^{-2}   ds 
	\end{align*}
	Now since $A\geq 1$, we have $\log( e AU / \sqrt r ) \geq  1$ and we can write 
$$  \log(   s AU   / \sqrt r  )  \leq \log( eAU   /\sqrt r )  ( 1 +  \log( s))$$ which implies that
\begin{align*}
 E_\sigma\left[R_n(\mathcal{F}_{n,r})\right]& \leq 12   \sqrt{  \frac{ vr  \log( eAU   /\sqrt r ) }  n }    \int_{  1 } ^\infty   s^{-2}  \sqrt{  1 + \log(  s  ) }     d s\\
  &  =  C     \sqrt{  \frac{ vr  \log( eAU   /\sqrt r ) }  n }  .
\end{align*}
\end{proof}

A further similar result is now given about the class 
$$\mathcal F_r = \{f \in \mathcal{F}: P(f^2) \leq r\}.$$
.

\begin{lemma}\label{lemma:sub-root-UB2}
Let $\mathcal{F}$ be a class of functions that is VC-type with envelope $U>0$ and parameter $v,A\geq 1$. We have, for any $r\leq U^2$,
$$ E[  R_n\{ \mathcal F_r \}] \leq   C  \sqrt { n^{-1}  v  r  \log( 5 AU   /  \sqrt r   ) }   +  8U C^2   n^{-1} v   \log(5 AU   /  \sqrt r   )       $$ 
where $C$ is defined in Lemma \ref{lemma:sub-root-UB}. Moreover if $  r  \geq n^{-1} U^2 $, we obtain
$$ E[  R_n\{ \mathcal F_r\}]\leq  C\sqrt{  v  n^{-1} r   \log(  5 A  \sqrt{ n }    )    } + 8 C^2   U v  n^{-1} \log(  5  A \sqrt{ n }  ) . $$ 

\end{lemma}

\begin{proof}
First we apply Lemma \ref{lemma:sub-root-UB} with the largest possible $r $ given by $ \hat \sigma_n^2 = \sup_{f\in \mathcal F_r} P_n (f^2)  $. Note that by definition $P_n(f^2) \leq \hat \sigma_n^2$ for all $f \in \mathcal F_r$, so that
$$ \mathcal F_r = \{ f\in \mathcal F_r \, : \, P_n (f^2 ) \leq \hat \sigma_n^2   \} $$
and we obtain 
$$E_\sigma [  R_n\{\mathcal F_r \}] \leq  C  \sqrt{     \hat \sigma_n^2    n^{-1} v  \log(  e AU    /  \hat \sigma_n    )  ] }.$$
Using twice Jensen inequality (functions $\sqrt x$ and $ a x\log(b/x) $ are both concave), we get
\begin{align*}
E[  R_n\{\mathcal F_r \}]    &  \leq  C  \sqrt{  E [  \hat \sigma_n^2    n^{-1} v  \log(  e AU    /  \hat \sigma_n    )  ] }\\
  & \leq C  \sqrt{   E [  \hat \sigma_n^2  ] n^{-1}  v  \log(  (eAU)^2   /  E [   \hat \sigma_n^2  ]  ) /2    }\\
  & \leq C  \sqrt{   E [  \hat \sigma_n^2   ] n^{-1}  v  \log( 9e  (AU)^2   /  E [   \hat \sigma_n^2  ]  ) /2    }
\end{align*}
where the last point is just convenient for the increasing function property which will be used in the next few line.
From Corollary 3.4 in \cite{10.1214/aop/1176988847}, we obtain 
\begin{align*}
  E [  \hat \sigma_n^2   ]  & \leq  r  + 8U  E [ R_n  ( \mathcal F_r)] .
\end{align*}
 Now remark that $x\log(b/x)$ is increasing for $x\leq  b /  e $.  This is always satisfied for $x = r + 8U n^{-1}   E [ R_n  ( \mathcal F_r)] $  because this quantity is smaller than $9U^2$ which itself is smaller or equal  to $ b / e $ when $b = 9 e  (AU)^2   $ because $A\geq 1$. We therefore obtain
\begin{align*}
  E[  R_n\{\mathcal F_r \}]  
  & \leq C\sqrt{ n ^{-1} v (r  + 8U     E [ R_n  ( \mathcal F_r)] ) \log(  9e (AU  )^2/( r  + 8U   E [ R_n  ( \mathcal F_r)] )  ]  /2 }\\
  & \leq C\sqrt{ n ^{-1} v  (r  + 8U   E [ R_n  ( \mathcal F_r)] )   \log(9e (AU   )^2 /  r  )  /2  }
  \end{align*}
  Therefore
  $$ E [ R_n  ( \mathcal F_r)]  ^2 \leq C^2 n^{-1} v  (  r+ 8U   E [ R_n  ( \mathcal F_r)]  )   \log( 3 \sqrt e  AU   /  \sqrt r   )   = b E [ R_n(\mathcal G) ]  +c $$
with $b =  8U C^2   n^{-1} v   \log( 3 \sqrt e  AU   /  \sqrt r   )    $ and $c = r C^2 n^{-1}  v   \log( 3 \sqrt e  AU   /  \sqrt r   )   $. It implies that 
$E [ R_n  ( \mathcal F_r)]  \leq b  +\sqrt c$ and the first statement follows by remarking that $3 \sqrt e \leq 5$.

\end{proof}

To proceed further, a central tool  
Theorem 3.3 in \cite{Bartlett2005} which we state below (Theorem~\ref{theo:fast-rates-ingredient} with the functional $T(f) = P(f^2)$ for which  $\Var(f)\leq T(f)$ and $T(\alpha f)\leq \alpha^2 T(f)$.

\begin{theorem}\label{theo:fast-rates-ingredient}
	Let $\mathcal{F}$ be a class of functions with envelope $U>0$ and suppose that $P(f^2) \leq  B P f$ for all $f\in \mathcal F$ for some $B\geq U$. Let $\psi$ be a sub-root function and let $r^\star$ be the fixed point of $\psi$, \ie $\psi(r^\star)=r^\star$. Assume that $\psi$ satisfies, for any $r \geq r^\star$,
	$$
	\psi(r) \geq  B E[  R_n\{f \in \mathcal{F}: P(f^2) \leq r\}]
	$$
	Then, for any $K>1$ and every $\delta>0$, with probability at least $1-\delta $,
	$$
	\forall f \in \mathcal{F} \quad P f \leq \frac{K}{K-1} P_n f+\frac{6 Kr^\star }{B} + \frac{\log(1/\delta) B \left(22  + 5   K\right)}{n}.
	$$
	Also, with probability at least $1 - \delta $,
	$$
	\forall f \in \mathcal{F} \quad P_n f \leq \frac{K+1}{K} P f+\frac{6 Kr^\star }{B} +  \frac{\log(1/\delta)B \left(22  +5  K\right)}{n}.
	$$ 
\end{theorem}
The next Proposition is key to obtain our main fast rates result Theorem~\ref{theo:fast-rates}. 
 \begin{proposition}\label{prop:simpleFastRateDeviation}
  Let $\mathcal{F}$ be a VC-type class of functions  with envelope $U>0$ and parameters $(\tilde A,\tilde v)$. Assume that $\mathcal{F}$ satisfies the Bernstein condition~\ref{cond:bernsteinF} with constant $B > U$ relative to a probability $P$ on $\mathcal{X}\times \mathcal{Y}$. Then with probability $1-\delta$,
   for all $f \in\mathcal{F}$, 
 \begin{align*}
   P  (f )\leq \frac{K}{K-1}P_{n }(f) 
   + \frac{ c_1 B  K \tilde v \log(  5\tilde A  \sqrt{ n }  / \delta)}{
   n } , 
 \end{align*}
 where $c_1>0$ is an explicit universal constant given in the
 proof.

 Also, with probability at least $1 - \delta $,
 $  \forall  f  \in {\mathcal{F}} $,
 \begin{align*}
   P_{n } (f ) 
   \leq \frac{K+1}{K}P (f ) 
   + \frac{ c_1 B  K    \tilde v \log(  5\tilde A  \sqrt{ n }  / \delta )}{
   n } .  
 \end{align*}

\end{proposition}
\begin{proof}
In light of the upper bound given in Lemma~\ref{lemma:sub-root-UB2}, and in order to apply Theorem~\ref{theo:fast-rates-ingredient}, 
 we introduce 
$$ \psi(r) =  b \sqrt {r} +c$$
with $ b = B C\sqrt{  \tilde v   n^{-1}    \log(  5\tilde A  \sqrt{ n }    )   }$ 
 and $c =   8 B^2   C^2    n^{-1} \tilde v  \log(  5\tilde A  \sqrt{ n }    )   $.
 The  function $\psi$, defined on $\mathbb R_+$, is sub-root with unique fixed point $r^*$ given by 
$
  \sqrt {r^*} = { (b+ \sqrt { b^2 +4c} ) / 2}    \leq b + \sqrt c$. Therefore
\begin{align}\label{ineq:fixed-point-UB}
   {r^*}   \leq  2 (b^2 +  c)  \leq   18 B^2    C^2  n^{-1} \tilde v \log(  5\tilde A  \sqrt{ n }    ) .
\end{align}
  We have that 
  $r\geq r^*$ implies that $ r \geq \psi(r) \geq c \geq  8 B^2 C^2 n^{-1} v  $ which is larger than $U^2 n^{-1}$  
  by our assumption on the constant $B$. Hence the second inequality in  Theorem~\ref{lemma:sub-root-UB2} holds true, which implies that whenever $ r\geq r^\star$, we have
$$ B E\left( R_n\{  f\in \tilde{\mathcal{F}}: P( f ^2) \leq r\}\right) \leq    B C\sqrt{  \tilde v   n^{-1}  r  \log(  5 \tilde A  \sqrt{ n }    )   } +    8 B U   C^2    n^{-1}  \tilde v  \log(  5 \tilde A  \sqrt{ n }    )   \leq  \psi (r)  .$$
Applying Theorem \ref{theo:fast-rates-ingredient} combined with the
previous bound for $r^*$, we obtain, with probability $1- \delta $,
\begin{align*}
\forall f \in  { \mathcal{F} }, \quad P (  f   )&\leq \frac{K}{K-1} P_n (  f  ) +   \frac{B S}{n} ,
\end{align*}
where 	
$$S =      c_1  K     \tilde v \log(  5  \tilde A  \sqrt{ n }    )       +  \log(1 /\delta)   (22 +5 K )  ,  $$
and $c_1 =18 \times 6 C^2  = 108 C^2$. 
It remains to check that, since $K>1$ and $v\geq 1$,
$$ S  \leq c_2  K    v \log(  5A  \sqrt{ n }  / \delta  ) ,      $$
with $c_2 = (c_1 \vee  27 ) = c_1$. We have established the first statement. For the second statement, the argument is similar, using this time  
the second statement of Theorem~\ref{theo:fast-rates-ingredient}. 
\end{proof}

\subsection{Proof of Theorem \ref{theo:fast-rates} }

We start by applying Proposition~\ref{prop:simpleFastRateDeviation} to the  
class
${\mathcal{F}}= \{(1-q) f_q I_+ + q f_q I_-,\; q\in(0,1),
f_q\in\mathcal{F}_q \}$  which, by assumption, is
 of VC-type with parameters $(\tilde A, \tilde v)$ and envelope $2U$.
 By construction, any $f\in\mathcal{F}$ writes as
  $f = (1-q) f_q I_+ + q f_q I_-$ for some $f_q\in\mathcal{F}_q$. 
  The first statement of Proposition~\ref{prop:simpleFastRateDeviation} writes in this context as follows:
  
  With probability at least $1-\delta$, for all $q\in(0,1)$, for all $f_q\in\mathcal{F}_q$, 
 \begin{align*}
   P \big((1-q) f_q I_+ + q f_q I_- \big) &\le    \frac{K}{K-1}P_{n }\big((1-q) f_q I_+ + q f_q I_- \big) +  \\
  &  \frac{ c_1 B  K \tilde v \log(  5\tilde A  \sqrt{ n }  / \delta)}{
   n}. 
 \end{align*}
 Dividing both sides by $2 q(1-q)$ yields
 
 \begin{align*}
   \frac{1}{2} \big(q^{-1}P (f_q I_+) + (1-q)^{-1}P( f_q I_-) \big)
   &\le    \frac{K}{K-1}  \frac{1}{2}\big(q^{-1} P_n(f_q I_+) + (1-q)^{-1}P_n( f_q I_-) \big) +  \\
  &  \frac{ c_1 B  K \tilde v \log(  5\tilde A  \sqrt{ n }  / \delta)}{
   2nq(q-1)}, 
 \end{align*}
 which means precisely, by definition of $P_q$ and $P_{n,q}$ (see Section~\ref{sec:background}), 
 \begin{align*}
 P_q f_q  
   &\le    \frac{K}{K-1} P_{n,q} +    \frac{ c_1 B  K \tilde v \log(  5\tilde A  \sqrt{ n }  / \delta)}{
   2nq(q-1)}.  
 \end{align*}
 \qed

 \subsection{Proof of Lemma \ref{lemma:bernstein-cond}}
 Fix $q\in(0,1)$ and consider the function
 $\varphi(g) = P( (1-q) \ell_g I_+ + q\ell_g I_- ), g\in\mathcal{G}$.  Notice that
 $\varphi(g) = 2q(1-q) \risk_q(g)$. The function $\varphi$  is a convex combination of the functions $\varphi_+(g) = P( \ell_g I_+)$ and
    $\varphi_-(g) = P (\ell_g I_- )$ with coefficients $(1-q, q)$. By assumption of the statement,  $\varphi_+$ and $\varphi_-$ are both strongly convex with respective parameters $ p\lambda, (1-p)\lambda$, meaning that
    for all $\alpha\in(0,1)$ and $g_1,g_2$, it holds that for $s\in\{+,-\}$
    \begin{align*}
&    \varphi_s(\alpha g_1 + (1-\alpha)g_2 )  \le \\ & \alpha \varphi_s(g_1) +
    (1-\alpha) \varphi_s(g_2) - \frac{\alpha(\alpha-1)}{2} \pi_s \lambda\|g_1- g_2\|^2, 
    \end{align*}
   with $\pi_{+} = p$ and $\pi_- = 1-p$.  By convex combination with coefficients $q,1-q$ we obtain 
    \begin{align*}
      \varphi(\alpha g_1 + (1-\alpha)g_2) & \le \alpha \varphi(g_1) + (1-\alpha)\varphi(g_2) 
        \frac{\alpha(\alpha-1)}{2} \left( ( 1-q)p + q(1-p)\right) \lambda\|g_1- g_2\|^2.
    \end{align*}
    
  Thus  $\varphi$ is $\lambda'$-strongly convex with $\lambda' = \left( (1-q)p + q(1-p)\right) \lambda$.  Now the score $g_{q}^*$ is a minimizer of $\risk_q$, whence it is also a minimizer of $\varphi=2q(1-q)\risk_q$.  
    By strong convexity and because is a minimizer of $\varphi$ 
    we obtain that for any $g\in\mathcal{G}$, 
    \begin{equation}
      P\big( 
      (1-q)(\ell_g - \ell_{g_q^*})I_+  +  q(\ell_g - \ell_{g_q^*})I_-
      \big)
      =    \varphi(g) - \varphi(g_q^*)  
      \ge \lambda' \|g -  g^*_q\|^2. 
    \end{equation}
    In other words for all $h \in \mathcal{H}$ of the form
    $h = h_{q,g }= (1-q)(\ell_g - \ell_{g_q^*})I_+ + q(\ell_g -
    \ell_{g_q^*})I_-$ we have
    \begin{equation}
      \label{eq:sc-tilde}
      P(h_{q,g}) \le \lambda' \|g -  g^*_q\|^2. 
    \end{equation}


    On the other hand, the second assumption of the statement is that
    \begin{equation*}
      \| (\ell_{g_1}  - \ell_{g_2}) I_s \|_{L_2(P)} \le \sqrt{\pi_s} L \|g_1- g_2\|, 
    \end{equation*}    
    Thus we may write for all $g_1,g_2\in\mathcal{G}$, 
    \begin{align*}
          \|(1-q)(\ell_{g_1} - \ell_{g_2})I_+  + q (\ell_{g_1} - \ell_{g_2}) I_-
           \|_{L^2(P)}
               & \le  (1-q) \| (\ell_{g_1}  - \ell_{g_2}) I_+ \|_{L_2(P)} +
              q \| (\ell_{g_1}  - \ell_{g_2}) I_- \|_{L_2(P)} \\
        & \le L \left((1-q)\sqrt{p} + q\sqrt{1-p} \right)  \|g_1- g_2\|. 
    \end{align*}
    In other words,    
    \begin{equation*}
      P\Big[  \big((1-q)(\ell_{g_1} - \ell_{g_2})I_+  + q (\ell_{g_1} - \ell_{g_2}) I_- \big)^2 \Big] \le (L')^2 \|g_1-g_2\|^2, 
    \end{equation*}
    with $L' =  L \left((1-q)\sqrt{p} + q\sqrt{1-p} \right)$.
    Choosing $g_1=g$ and $g  = g_q^*$ yields that for any $h= h_{q,g}\in\mathcal{H}$ as above, 
    \begin{equation}
      \label{eq:lipschitz-tilde}
      P (h_{q,g}^2) \le L'^2 \| g - g_q^*\|^2.
    \end{equation}
    Combining~(\ref{eq:sc-tilde}) and~(\ref{eq:lipschitz-tilde}), we obtain  
    \begin{equation*}
      P (\tilde \ell_{q,g} - \tilde \ell_{q, g_q^*} )^2 \le \frac{L'^2}{\lambda'}
       P(\tilde \ell_{q,g} - \tilde \ell_{q, g^*_q} ). 
     \end{equation*}
     Now by Jensen inequality applied to the function $t\mapsto t^2$ and the convex combination with coefficents $(1-q), q$,  we have
     \begin{align*}
       (L')^2 & \le L^2 \left( (1-q) \sqrt{p}^2 + q\sqrt{1-p}^2 \right), 
     \end{align*}
     so that $(L')^2/\lambda' \le L^2/\lambda$, 
     which concludes the proof. 
     \subsection{Proof of Lemma~\ref{lemma:VC-tildeF}}
       Write
  \begin{align*}
    {\mathcal{ H}} &= \{ q (\ell_g - \ell_{g^*_q})I_+ + (1-q) (\ell_g - \ell_{g^*_q})I_- \}
      \subset \{ q f  + (1-q) g , q\in(0,1), f\in \mathcal{L}'_+, g \in \mathcal{L}'_- \} = \mathcal{F}'
  \end{align*}
  where $\mathcal{L}'_s = \{ (\ell_1 - \ell_2)I_s \, \; \ell_1,\ell_2\in\mathcal{L} \}$.
  From Lemma~\ref{lemma_vc_pres}(1., 2.), each $\mathcal{L}'_s$ is of VC-type with envelope $2U$ and parameters $(2v, 2A)$.  Using the $4^{th}$ statement of the same lemma, also $\mathcal{F}'$ is of VC type with envelope 2U and parameters $(2(2v+1),3*2A )$.

 \section{Numerical experiments: Real world dataset}

Our aim, just as in the main paper, is to illustrate the decision boundary of the $k$-nn classifiers on real-world datasets. To do so, we follow the same procedure as the main paper, but instead of using synthetic data, we employ six real-world datasets (Pima, Breast, Cardio, Sattelite, Annthyroid, Ionosphere) from the ODDS repository\footnote{http://odds.cs.stonybrook.edu}. Figures \ref{fig:breast} to \ref{fig:satellite} display the balanced accuracy ($1-\risk^{0-1}_p$) of the balanced $k$-nn as function of $(k,p)$, we make the proportion of positive class $p$ vary by randomly removing positive examples. Similar to the findings on synthetic data, these experiments suggest that a large number of neighbors $k$ should be chosen relative to $p = p_n$ to ensure the consistency of the nearest neighbors method.
It's important to note, however, that the learning boundary appears somewhat more noisy than in the synthetic data case. This is indeed not surprising since the number of examples available is significantly smaller in comparison to the previous simulation.
\begin{figure}[!h]
	\minipage{0.45\textwidth}
	\includegraphics[width=\linewidth]{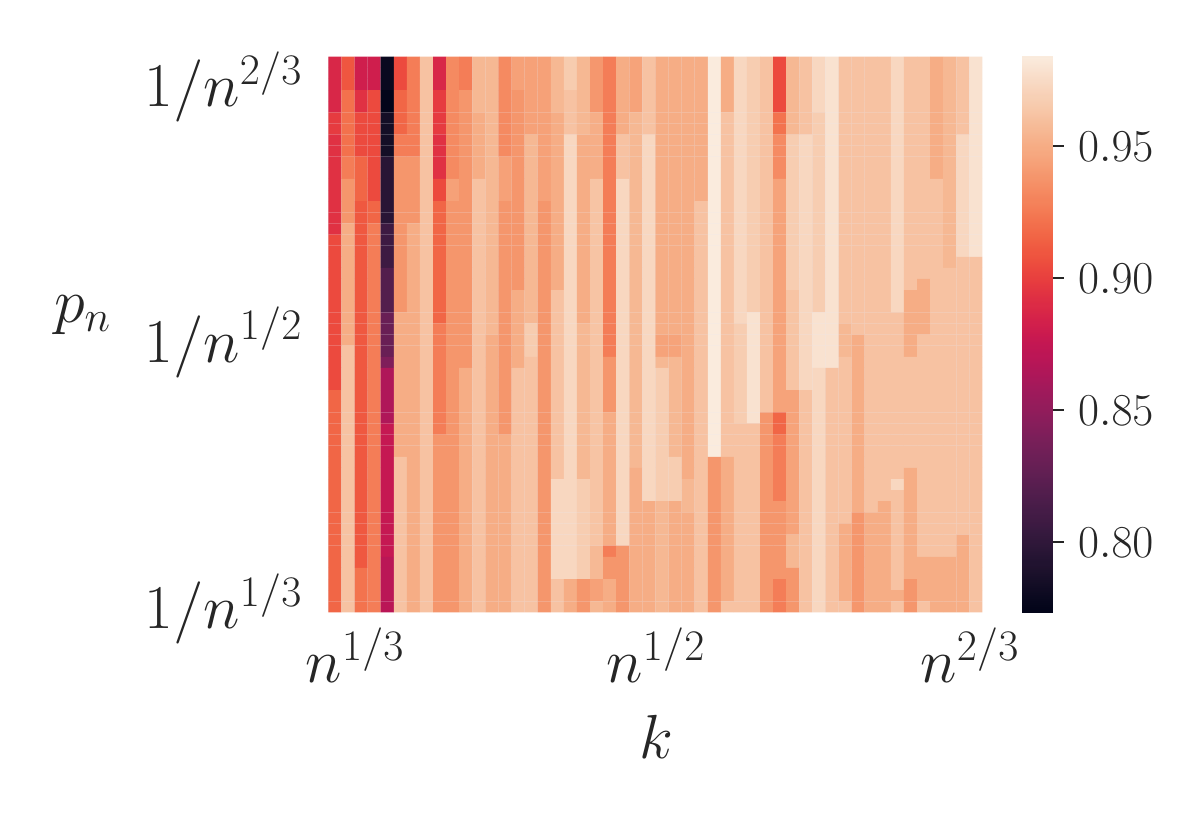}
	\caption{ Balanced accuracy heat map for the Breast dataset.}\label{fig:breast}
	\endminipage
	\hfill
	\minipage{0.45\textwidth}
	\includegraphics[width=\linewidth]{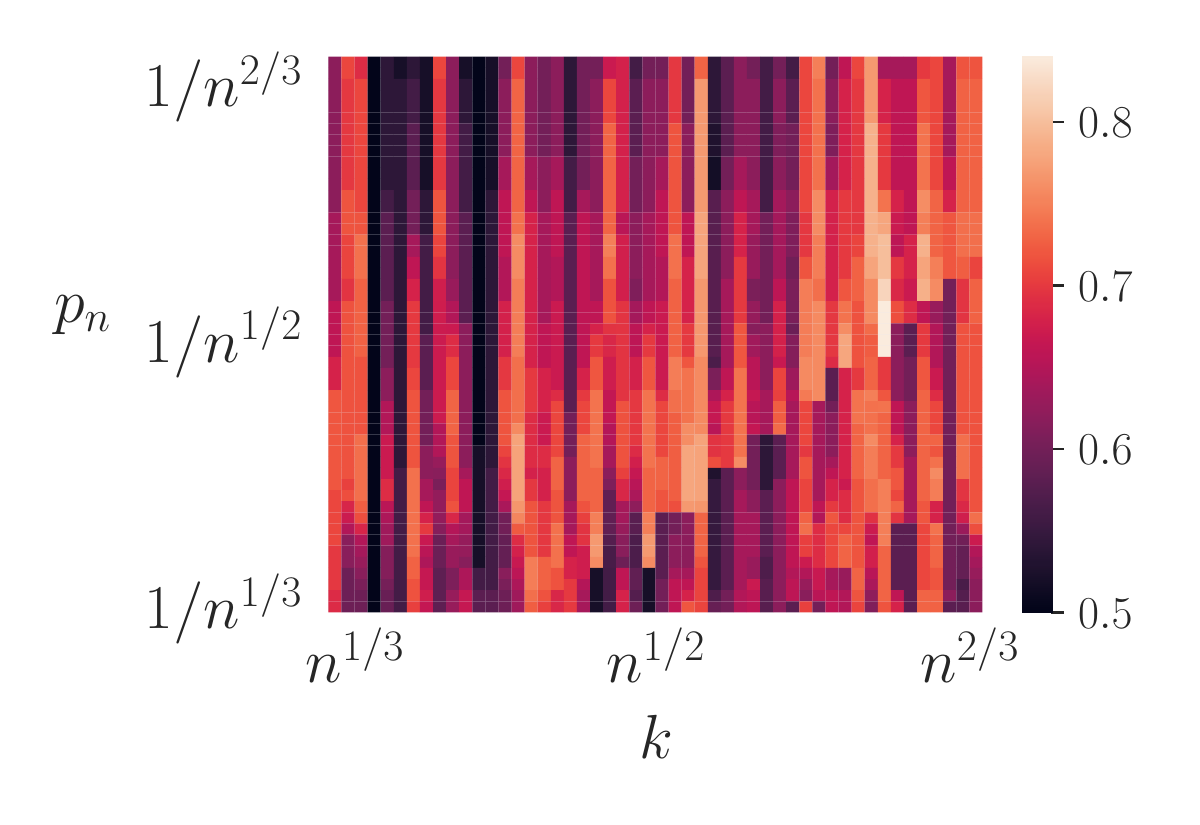}
	\caption{Balanced accuracy heat map for the Ionosphere dataset.}\label{fig:ionosphere}
	\endminipage
\end{figure}
\begin{figure}[!h]
	\minipage{0.45\textwidth}
	\includegraphics[width=\linewidth]{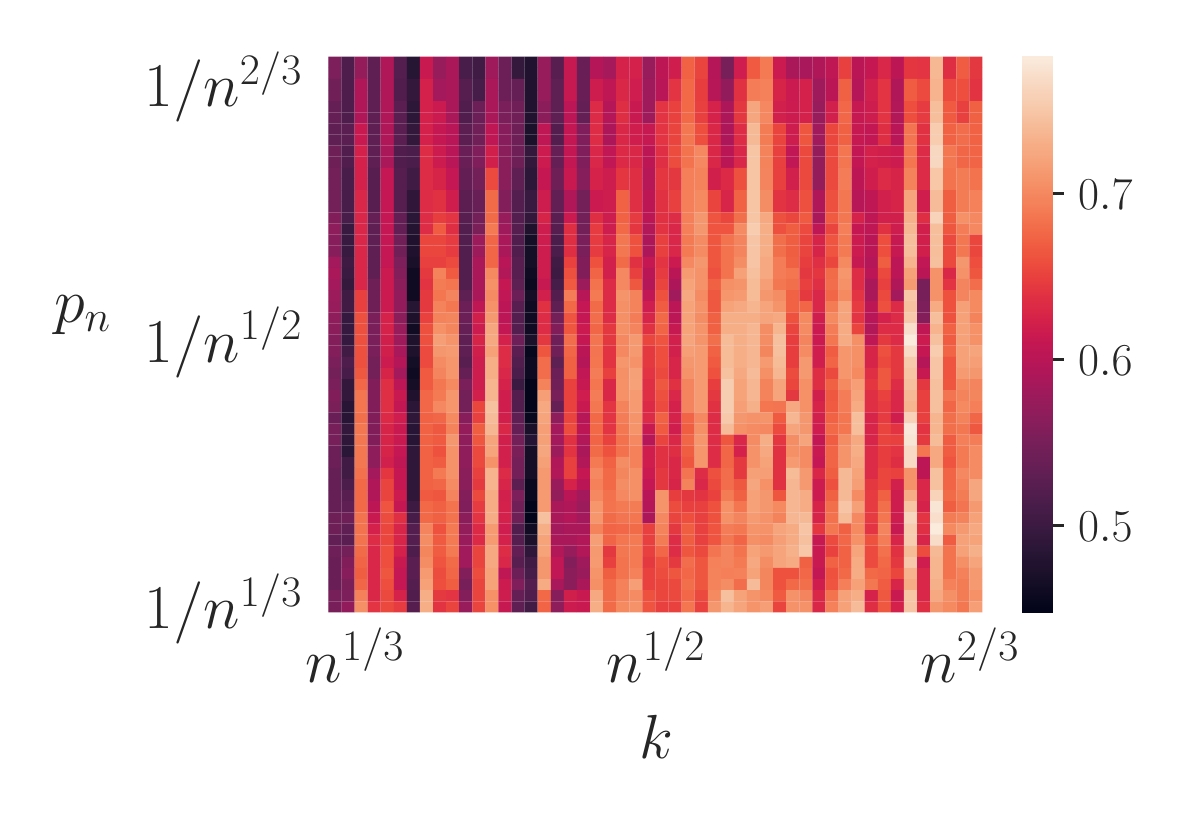}
	\caption{ Balanced accuracy heat map for the Pima dataset.}\label{fig:pima}
	\endminipage
	\hfill
	\minipage{0.45\textwidth}
	\includegraphics[width=\linewidth]{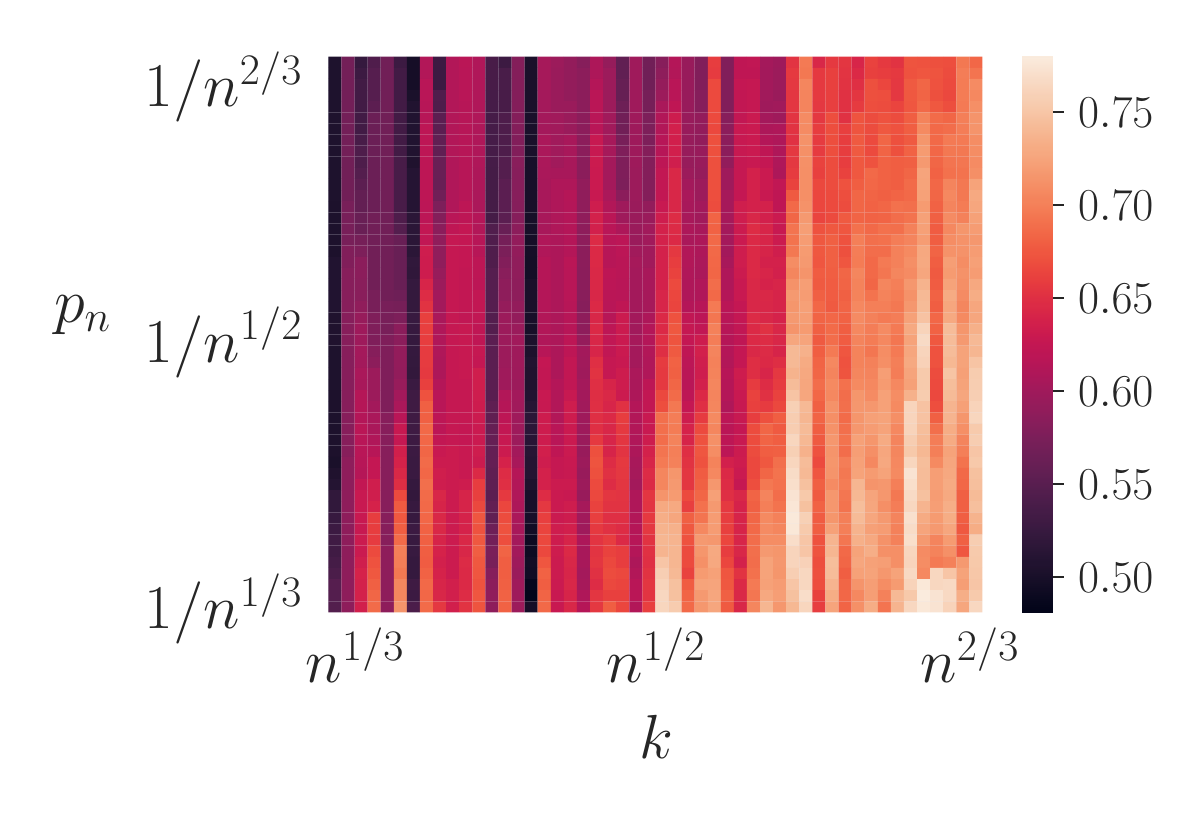}
	\caption{Balanced accuracy  heat map for the Annthyroid dataset.}\label{fig:annthyroid}
	\endminipage
\end{figure}
\clearpage
\begin{figure}
	\vspace{-160mm}
	\minipage{0.45\textwidth}
	\includegraphics[width=\linewidth]{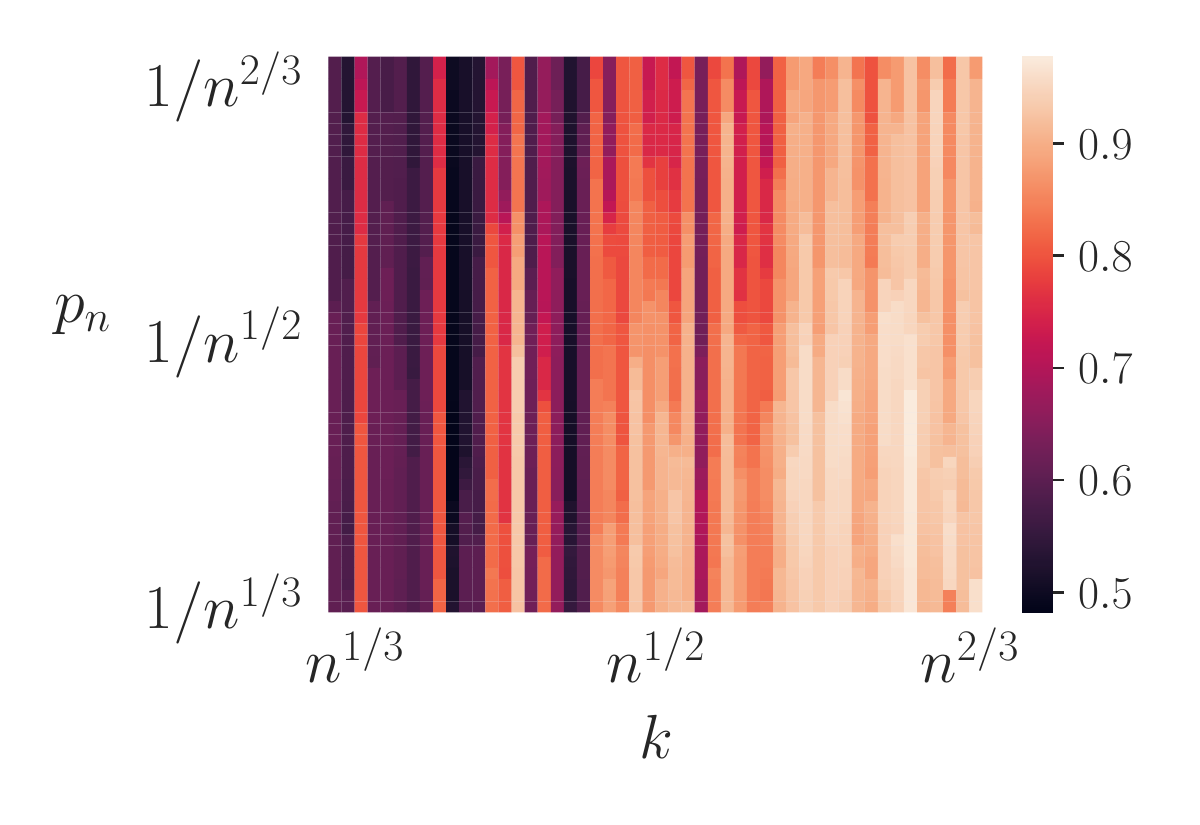}
	\caption{Balanced accuracy  heat map for the Cardio dataset.}\label{fig:cardio}
	\endminipage
	\hfill
	\minipage{0.45\textwidth}
	\includegraphics[width=\linewidth]{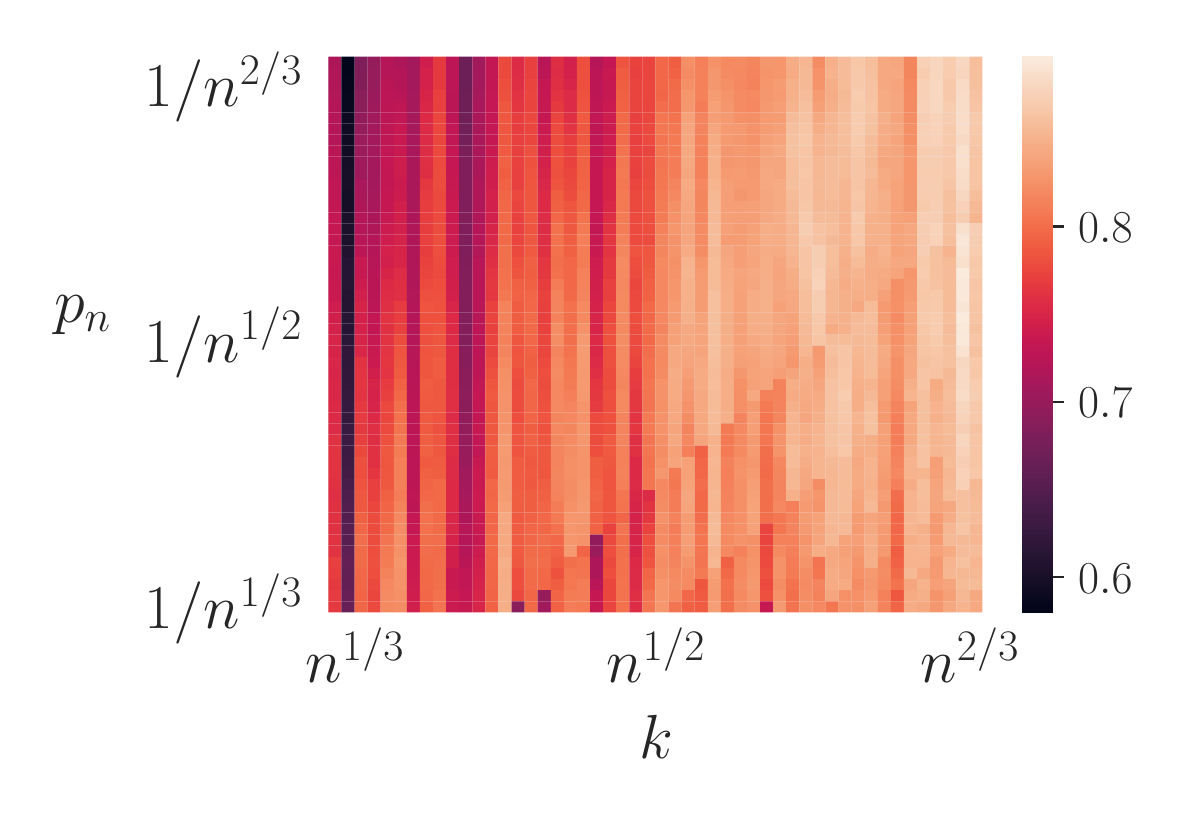}
	\caption{Balanced accuracy   heat map for the Satellite dataset.}\label{fig:satellite}
	\endminipage
\end{figure}
 	

\end{document}